\documentclass[a4paper,11pt,reqno]{article}

\usepackage{graphicx,color,subfigure,multirow, here}
\usepackage{mathtools}
\usepackage[utf8]{inputenc} 
\usepackage[T1]{fontenc}
\usepackage[english]{babel}
\usepackage{amsmath, amsfonts, amsthm,amssymb,here,dsfont, hyperref}
\usepackage{graphicx,color,subfigure,multirow, here}
\usepackage{bm}
\usepackage{natbib} 
\usepackage{mathtools} 
\usepackage{csquotes} 
\usepackage{algpseudocode}
\usepackage[linesnumbered,ruled,vlined]{algorithm2e}
\usepackage{soul}
\usepackage{comment}
\usepackage[normalem]{ulem} 
\DeclareMathOperator*{\argmax}{arg\,max} 
\DeclareMathOperator*{\argmin}{arg\,min} 
\newtheorem{theo}{Theorem}[section]

\newtheorem{prop}[theo]{Proposition}

\newtheorem{lem}[theo]{Lemma}
\newtheorem{rem}[theo]{Remark}
\newtheorem{defi}{{Definition}}[section]

\newtheorem{ass}[theo]{Assumption}

\makeatletter

\@addtoreset{equation}{section}
\makeatother
\DeclareMathOperator*{\sign}{sign}

\usepackage{graphicx}
\usepackage{tikz}
\usetikzlibrary{arrows}

\tikzset{
  treenode/.style = {align=center, inner sep=0pt, text centered,
    font=\sffamily},
  arn_n/.style = {treenode, circle, white, font=\sffamily\bfseries, draw=black,
    fill=black, text width=1.8em},
  arn_g/.style = {treenode, circle, black, font=\sffamily\bfseries, draw=green,
    fill=green, text width=1.8em},
  arn_r/.style = {treenode, circle, black, font=\sffamily\bfseries, draw=black,
    fill=red, text width=1.8em},
  arn_x/.style = {treenode, rectangle, draw=black,
    minimum width=0.5em, minimum height=0.5em}
}
\begin{document}
\title{Nonparametric active learning for cost-sensitive classification.}
\date{\today}
\author{
Boris Ndjia Njike, Xavier Siebert\\
Mathematics and Operational Research,
University of Mons, Belgium\\
}

\maketitle

\begin{abstract}%
  Cost-sensitive learning is a common type of machine learning problem where different errors of prediction incur different costs. In this paper, we design a generic nonparametric active learning algorithm for cost-sensitive classification. Based on the construction of confidence bounds for the expected prediction cost functions of each label, our algorithm sequentially selects the most informative vector points. Then it interacts with them by only querying the costs of prediction that could be the smallest. We prove that our algorithm attains optimal rate of convergence in terms of the number of interactions with the feature vector space. Furthermore, in terms of a general version of Tsybakov's noise assumption, the gain over the corresponding passive learning is explicitly characterized by the probability-mass of the boundary decision. 
  Additionally, we prove the near-optimality of obtained upper bounds by providing  matching (up to logarithmic factor) lower bounds.
\end{abstract}
\hspace{0.80cm}
  
\section{Introduction}
\label{sec:intro}
For many real-world machine learning tasks, while unlabelled data are abundant, getting a pool of labelled data is very expensive and time-consuming. In this case, it is interesting to only label the points that could significantly affect the prediction decision. This is the main purpose of \textit{active learning} \citep{cohn1994improving,dasgupta2011two} which is a machine learning approach that attemps to provide an optimal decision rule while using as few labelled data as possible. Contrary to standard models of machine learning (namely \textit{passive learning}) where the labelled data are provided beforehand, in \textit{active learning}, the learner only has access to a set of unlabelled points at the beginning, and  has to progressively request (to a so-called \textit{oracle}) at some cost the label of some points during the learning process. At the end, based on these selected labelled points, a prediction decision rule is provided. In this paper, we consider the classification task which consists in providing a prediction rule or classifier that maps each unlabelled point $x$ from $\mathcal{X}$, the \textit{instance space}, to  an element $y$ $\in$ $\mathcal{Y}=\{1,\ldots,M\}$ the \textit{label space} $($where $M\geq 2$ is an integer$)$. Both in active and in passive learning, the performance of a prediction rule is measured in terms of its ability to predict the label of a new instance by keeping the error of classification as small as possible. However, while fruitful in many domains applications, measuring the performance of a classifier by only considering its ability to maintain the error of classification as small as possible is sometimes inappropriate because some type of classification errors (even small) could have significant negative effects on the underlying problem. For many real world domain applications, it is thus valuable to consider learning classification approaches that pay attention to each error of classification. For example, in fraud detection tasks, a classification error that deems a fraudulent transaction as legitimate may incur huge losses for the banking institution, such as financial or reputational, while refusing a legitimate transaction is less serious because the banking institution only loses the transaction fee.     
The learning classification problem can be reformulated in terms of the cost incurred by a label prediction, so that a prediction rule will predict the label which has the smallest negative effect on the underlying problem. Such learning problem is called \textit{cost-sensitive classification} \citep{elkan2001foundations,dmochowski2010maximum}. Let us emphazise that the prediction cost differs from the labeling cost mentioned above in the definition of active learning. \textit{Labeling cost} in active learning corresponds to the effort (for example, it could be expressed in terms of time) made to determine the label of an instance while \textit{prediction cost} in cost-sensitive classification refers to the cost incurred by making some prediction as in the fraud detection example.

In this paper, we design an active learning algorithm that provides a classifier suitable for cost-sensitive classification. Our analysis falls into the nonparametric setting under assumptions related to those used in \citep{audibert2007fast, minsker2012plug,locatelli2017adaptivity,kpotufe2022nuances}. At each step, our algorithm is able to select any point from the instance space and to interact with it by requesting the prediction costs that could be the smallest (with high probability). Under some smoothness and margin noise assumptions, we prove that the resulting classifier achieves an optimal rate of convergence in terms of the number of interactions with the instance space. Additionally, we show the (near) optimality of our rate of convergence by providing matching (up to a logarithmic factor) lower bounds. These results compare favorably with those obtained in the passive learning counterpart \citep{reeve2019learning}.
The paper is organized as follows: Section~\ref{sec:relatedWorks} presents backgrounds and related works on active learning, cost-sensitive classification and active learning for cost-sensitive classification. In Section~\ref{sec:contributions} we summarize our main contributions while in Section~\ref{sec:algo} we introduce the main notations, assumptions used in this work, the overall description of our algorithm and the theoretical properties of our algorithm. We conclude with 
Section~\ref{sec:conclusion}. Section \ref{sec:appendix} contains all detailed proofs.
\section{Related works}
\label{sec:relatedWorks}

\textbf{Active learning}. 
Over the last fifteen years, a lot of progress has been made on understanding the principles of active learning, and its benefits over passive learning (see for example \citep{balcan2007margin,dasgupta2011two,castro2008minimax,minsker2012plug,hanneke2015minimax,locatelli2017adaptivity,kpotufe2022nuances}). 
Some of the studies considered the nonparametric setting by focusing on plug-in methods in classification \citep{minsker2012plug,locatelli2017adaptivity}. The overall principle  is to progressively estimate the class-conditional regression functions by only focusing in the region where the classification is most difficult. Under some standard smoothness and noise assumptions, \citep{minsker2012plug, locatelli2017adaptivity} provided a rate of convergence that shows the improvement (over some range of smoothness and noise parameters) of active learning over passive learning \citep{audibert2007fast}. A deep analysis of these bounds and of their the proof details reveals that these improvements over passive learning strongly depend on the behavior of the class-conditional regression functions around the boundary decision. Very recently, \citep{kpotufe2022nuances} went further by introducing a more general noise assumption than that used in \citep{minsker2012plug,locatelli2017adaptivity}. This assumption takes into account the probability-mass of the boundary decision and allows to prove that no improvement is possible over passive learning when there is a large amount of instances exactly on the decision boundary. The overall principle of our algorithm is inspired by \citep{minsker2012plug,locatelli2017adaptivity,shekhar2021active}, and is adapted for the cost-sensitive classification. When interacting with the oracle, the prediction costs are requested instead of the label as was done in \citep{minsker2012plug,locatelli2017adaptivity}. Moreover, we use a more general noise assumption, and depending on how large the probability-mass of the boundary decision is, we explicitly identify when a gain is possible in cost-sensitive active learning over its passive learning counterpart \citep{reeve2019learning}.  

\textbf{Cost-sensitive learning}. As mentioned in Section \ref{sec:intro}, formulating the objective only in terms of the error of classification is inappropriate in some practical situations. Some works have pointed out this disadvantage by considering the cost-sensitive approach to classification \citep{elkan2001foundations,dmochowski2010maximum,reeve2017minimax, reeve2019learning} where the costs incurred by each error of classification are taken into account. Since the pioneering work by \citep{elkan2001foundations}, a plethora of cost-sensitive learning approaches have been proposed, but we consider here three that are closely related to our present setting.
\\
First, \citep{elkan2001foundations} introduced a cost-sensitive approach where an $M\times M$-matrix $\Gamma$ is given to the learner before the learning process. The matrix $\Gamma$ represents a cost matrix, and $\Gamma_{ij}\geq 0$ represents the cost incurred by predicting the label $i$ whereas the true label is  $j$. When $\Gamma_{ij}=1$ if $i\neq j$ and $\Gamma_{ii}=0,$ it reduces to the standard approach where the cost of misclassification is not taken into account. The learning objective is thus to find a prediction rule that minimizes the expected cost with respect to the cost matrix $\Gamma$. This approach was also considered by other works such as \citep{reeve2017minimax, reeve2019learning}. These works were built upon \citep{audibert2007fast}, which considered the standard learning approach by implicitly making the assumption of symmetric prediction costs. 
\\
The second cost-sensitive approach consists in making $\Gamma$  a feature-dependent matrix cost \citep{elkan2001foundations, reeve2019learning}. This makes sense in some real-world applications such as medical diagnosis where the prediction cost may depend on the patient (medical history for instance). 
\\
The third cost-sensitive approach, which will be considered in this paper,  is more general than the two previous ones. It considers  the case where the cost matrix $\Gamma$ $($or $\Gamma(x))$ is not available to the learner. This is very common in practice such as medical diagnosis where it is sometimes very difficult to evaluate the cost that would have been incurred if an unhealthy patient was predicted as healthy. In this case, an instance $X$ does not come directly to the learner with a label $y$, but, it is associated with a (random) cost-vector  of size $M$, where the $i$-th component represents the prediction cost incurred by the label $i$. 

The three aforementioned cost-sensitive approaches were considered in the case of passive learning \citep{reeve2019learning}. In the third approach, they advocated for maximizing rewards rather than minimizing costs as in the present work. In nonparametric passive learning setting, they studied minimax learning rate under some smoothness and noise assumptions which are closely related to ours. Their analysis and results appear to be more general than the work of \citep{audibert2007fast}. In this paper, we study for the first time the cost-sensitive classification problem in nonparametric active learning setting in the same line of research as \citep{audibert2007fast} and compare our statistical guarantees results to the corresponding passive learning setup \citep{reeve2019learning}.  
\\
\textbf{Active learning for cost-sensitive classification}.
Cost-sensitive multi-class classification in active learning has received attention in the last decade mostly in parametric setting \citep{agarwal2013selective,krishnamurthy2017active,JMLR:v20:17-681}. Through the online selective sampling framework, \citep{agarwal2013selective} designed an algorithm for cost-sensitive multiclass-classification and showed improvement of active learning over passive learning. However, they considered learning models where the query strategies are tailored to linear representation \citep{gentile2012multilabel} which is too specific and difficult to extend to other hypothesis classes. Additionally, they used the cost-sensitive setting where the cost-matrix is fixed and available to the learner before the learning process. Later, \citep{krishnamurthy2017active} considered active learning for cost-sensitive multi-class classification in parametric setting and where the prediction costs are unknown. For every considered example $X$, their algorithm computes upper and lower bounds of the expected prediction cost of each label and then only queries labels with relatively large prediction range. However, their computational techniques leads to sub-optimal results in label complexity. This latter drawback was overcome in \citep{JMLR:v20:17-681}, where an efficient cost range computation is provided and leads to optimal results in terms of label complexity. In this paper, we rather consider the nonparametric setting. We derive new techniques  for computing with high confidence, the upper and lower bounds on the expected prediction costs. This allows us to make optimal predictions and then provide a classifier suitable for cost-sensitive classification that achieves optimal rate of convergence.   

\section{Contributions}
\label{sec:contributions}
In this present work, we propose a generic nonparametric active learning algorithm for cost-sensitive multiclass-classification by considering a more general noise assumption than that used in earlier works such as \citep{reeve2019learning}. Our contributions are as follows:
\begin{itemize}
    \item Our resulting classifier achieves an optimal rate of convergence which involves a noise parameter that allows to precisely determine when gain is possible over the passive learning counterpart \citep{reeve2019learning}. Moreover, our algorithm is adaptive to this noise parameter.
    \item We provide a lower bound that matches (up to a logarithmic factor) the obtained rate of convergence achieved by our algorithm. 
    \end{itemize}
    \section{Algorithm}
\label{sec:algo}
In this Section, we first introduce some notations, then we provide a brief description and analysis of our active learning algorithm. For clarity purposes, some key quantities that will be used in the analysis of our algorithm are summarized in Table \ref{table:1}. We conclude by giving some theoretical results that show that our active learning algorithm enjoys good statistical guarantees. 
\subsection{Setting}
\label{sec:setting}
Let $\mathcal{X}\subset \mathbb{R}^d$ be a metric space, referred to as the \textit{instance space}. We assume that $\mathcal{X}=[0,1]^d$ is equipped with the Euclidean distance $\|.\|.$ Let $\mathcal{Y}=\{1,\ldots,M\}$ be the set of \textit{labels}.  Let $P$ be a (unknown) probability defined on $\mathcal{X}\times [0,1]^M.$ In this case, if $(X,c)\sim P$, with $c=(c(1),\ldots,c(M))$, then for $y=1,\ldots,M$, the quantity $c(y)$ represents the prediction cost induced by the use of label $y$. In cost-sensitive learning, instead of dealing with a labelled sample from $\mathcal{X}\times\mathcal{Y}$, the learner observes an i.i.d sample $(X_1,c_1), \ldots, (X_n,c_n)$ generated according to $P.$ The random quantity $c_k$ $(k=1,\ldots,n)$ represents the cost vector $c_k=(c_k(1),\ldots,c_k(M))$. In this context, the main objective is to construct a mapping  $g:\mathcal{X}\longrightarrow \mathcal{Y}$ which minimizes the cost-sensitive expected cost:
$$R_{cs}(g)=\mathbb{E}_{(X,c)}(c(g(X))).$$ 

Let $y \in \mathcal{Y}$ be a given label, and the function: $f(.;y): \mathcal{X}\longrightarrow [0,1]$ such that $f(x;y)=\mathbb{E}(c(y)\vert X=x)$.
The oracle mapping $f_{cs}^*:\mathcal{X}\longrightarrow \mathcal{Y}$ defined by 
$$ f_{cs}^{*}(x)\in\argmin_{y\in \mathcal{Y}}\;f(x;y)$$ and known as the \textit{Bayes cost-sensitive classifier}, 
achieves the minimum cost-sensitive expected cost  \citep{JMLR:v20:17-681}. The excess expected cost for a mapping $g: \mathcal{X}\longrightarrow \mathcal{Y}$ is: 
$$\varepsilon_{cs}(g)=R_{cs}(g)-R_{cs}(f^*_{cs})=\mathbb{E}_{X}\left(f(X;g(X))-f(X;f_{cs}^*(X))\right).$$
\begin{rem}
Cost-sensitive learning is more complex than standard classification problem. In fact, in cost-sensitive learning, labels just serve as indices for the cost vector, and the distribution of data is over pairs of $(x, c)$ rather than $(x, y)$ as seen in (active) binary and multiclass classification \citep{kpotufe2022nuances}. Moreover, cost-sensitive decision-making involves minimizing the expected cost of a decision, which may involve predicting a label that is less probable (with respect to the standard classification) but has a lower prediction cost associated with it. However, in certain specific cases, a cost-sensitive problem can be transformed into a classification problem. (see for instance Proposition 5.8.1 from \citep{reeve2019learning}). 
\end{rem}
\subsection{Our model} 
\label{sec:model}
We consider a hierarchical structure on the instance space:
 for integers $h\geq 0$, $K\geq 2$, the instance space is partitioned into $K^{h}$  subsets $\mathcal{X}_{h,i}$ $(1\leq i\leq K^h)$ which are referred to as \textit{cells}.
 This partitioning  is seen as a spatial $K$-tree where the root is the whole space $\mathcal{X}=\mathcal{X}_{0,1}$ and at each depth $h\geq 0$, we have $\displaystyle\cup_{i=1}^{K^h}\mathcal{X}_{h,i}=\mathcal{X}$  and each cell $\mathcal{X}_{h,i}$ induces $K$ children $\lbrace \mathcal{X}_{h+1, i_j}, j=1,\ldots, K\rbrace$ which forms a partition of $\mathcal{X}_{h,i}$. 
 Additionally, for each cell $\mathcal{X}_{h,i}$, we associate a fixed point $x_{h,i}$ $\in$ $\mathcal{X}_{h,i}$ called the center of $\mathcal{X}_{h,i}$ such that
    $\mathcal{X}_{h,i}=\lbrace x\in \mathcal{X},\; \|x-x_{h,i}\|\leq \|x-x_{h,j}\|,\;\;j\neq i\rbrace$ and where ties are broken in an arbitrary, but deterministic way. Moreover, we suppose: 

\begin{itemize}
    \item there exists some constants $\rho\in (0,1)$, $\nu_1,\nu_2>0$, with $0<\nu_1\leq 1\leq \nu_2$ such that for all $h,i$: 
    \begin{equation}
    \label{eq:ball-ass1}
    B(x_{h,i}, \nu_1\rho^h)\subset \mathcal{X}_{h,i}\subset B(x_{h,i}, \nu_2\rho^h).
    \end{equation}
    where $B(x,r)=\lbrace z\in \mathcal{X},\;\|x-z\|< r\rbrace$ for $x \in \mathcal{X}$, and $r>0$.
    \item that $K=2$, and $\mathcal{X}_{h,i}$ is partitioned as
    $\mathcal{X}_{h,i}=\mathcal{X}_{h+1,2i-1}\cup \mathcal{X}_{h+1,2i}.$
     
  \end{itemize} 
A similar model was also considered in different settings such as bandit theory \citep{munos2014bandits}.
In this paper, we consider an active learning algorithm for cost-sensitive multi-class classification that exploits the hierarchical structure of the instance space. Given a hierarchical partition of the instance space,  and some integers $h$ and $i$, we aim at estimating the Bayes cost-sensitive classifier in $\mathcal{X}_{h,i}$ and thereby in the whole space. Our algorithm proceeds by only interacting with cells $\mathcal{X}_{h,i}$ $($for some $h, i)$ where there is substantial classification uncertainty. It interacts with $\mathcal{X}_{h,i}$ via its center $x_{h,i}$ by requesting some of its label costs. In the following sections, we will sometimes consider $x_{h,i}$ as the cell $\mathcal{X}_{h,i}$ if no ambiguity is present.

One of the main reasons for which we use the hierarchical structure on the instance space in this paper is to align with the majority of previous works on standard nonparametric active learning, such as \citep{castro2008minimax,minsker2012plug,locatelli2017adaptivity,kpotufe2022nuances}.  However, We could also use the pool-based approach considered in \citep{njike2022multi} where a k-NN active learner has been proposed. 
\subsection{Overall description of our algorithm}
\label{sec:algorithm description}
With a fixed number of interactions $n$ (considered as the budget), our main objective is to provide an active learning algorithm for cost-sensitive multi-class classification, which outputs a classifier with good statistical guarantees.  Our algorithm is inspired from that provided in \citep{shekhar2021active} in the context of active learning with reject option,  and also from that provided in \citep{JMLR:v20:17-681} where a bookkeeping procedure is used in order to only deal with labels that could be the optimal one.\\ 
Our algorithm works iteratively over a finite number of steps, until the budget $n$ has been reached.\\ 
At step $t$, the instance space is hierarchically partitioned into two subsets: $\mathcal{X}^{(t)}_u$ the set of unclassified cells (for example, filled in red on Figure \ref{fig-hierarchical}), and $\mathcal{X}^{(t)}_c$ the set of classified cells (for example, filled in green on Figure \ref{fig-hierarchical}). 
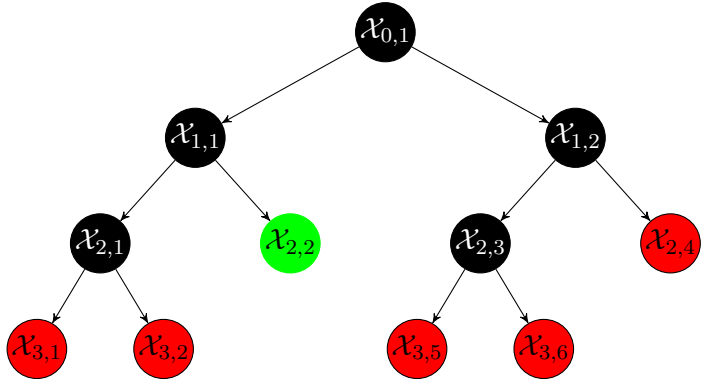
\begin{figure}[h]
\centering
\begin{tikzpicture}[->,>=stealth',level/.style={sibling distance = 5cm/#1,
  level distance = 1.4cm}] 
\node [arn_n] {$\mathcal{X}_{0,1}$ {}}
    child{ node [arn_n] {$\mathcal{X}_{1,1}$ {}} 
            child{ node [arn_n] {$\mathcal{X}_{2,1}$ {}} 
            	child{ node [arn_r] {$\mathcal{X}_{3,1}$ {}} edge from parent node[above left] {}
                         } 
							child{ node [arn_r] {$\mathcal{X}_{3,2}$ {}} edge from parent node[above left] {}}
            }
            child{ node [arn_g] {$\mathcal{X}_{2,2}$ {}}
							edge from parent node[above left] {}}                            
    }
    child{ node [arn_n] {$\mathcal{X}_{1,2}$ {}}
            child{ node [arn_n] {$\mathcal{X}_{2,3}${}} 
							child{ node [arn_r] {$\mathcal{X}_{3,5}$}{}}
							child{ node [arn_r] {$\mathcal{X}_{3,6}$} {}}
            }
            child{ node [arn_r] {$\mathcal{X}_{2,4}$ {}} edge from parent node[above left] {}}
            };
\end{tikzpicture}
\caption{An example of a possible partition of the instance space at a given step: the classified regions are represented by the cell filled in green, and the unclassified region by the cells filled in red: 
$\mathcal{X}^{(t)}_u=\{\mathcal{X}_{2,4},\mathcal{X}_{3,1},\mathcal{X}_{3,2},\mathcal{X}_{3,5},\mathcal{X}_{3,6}\}, \; \mathcal{X}^{(t)}_c=\{\mathcal{X}_{2,2}\}$ and $\mathcal{X}=\mathcal{X}^{(t)}_u\cup \mathcal{X}^{(t)}_c$.}
\label{fig-hierarchical}
\end{figure}

Subsequently, it proceeds as follows: it selects the cell $\mathcal{X}_{h_t,i_t}$ in $\mathcal{X}^{(t)}_u$ that has the largest classification uncertainty. In this case, two outcomes are possible: 
\begin{itemize}
\item First, it could interact with $\mathcal{X}_{h_t,i_t}$ by requesting the cost $c_{h_t,i_t}(y)$ of some labels at the center $x_{h_t,i_t}$. Based on these later requests $($and possibly other requests made before the step $t$),  a subset $\mathcal{Y}_{h_t,i_t}^{(t)}\subset\mathcal{Y}$ of candidate labels is built. This is done by discarding suboptimal labels $y$, that is those with $f(x;f_{cs}^*(x))<f(x;y)$ for all $x\in \mathcal{X}_{h_t,i_t}$ with high probability. Thus, if the subset $\mathcal{Y}^{(t)}_{h_t,i_t}$ only contains one label, the cell $\mathcal{X}_{h_t,i_t}$ is added to  the current classified region $\mathcal{X}^{(t)}_c$ and removed from $\mathcal{X}^{(t)}_u$.

\item Second, the cell $\mathcal{X}_{h_t,i_t}$ could be expanded or refined. This happens when the algorithm has interacted with $\mathcal{X}_{h_t,i_t}$ a substantial number of times. In this case, the cell $\mathcal{X}_{h_t,i_t}$ is replaced by $\mathcal{X}_{h_t+1,2i_t}\cup\mathcal{X}_{h_t+1,2i_t-1}$  and removed (for example, filled in black on Figure \ref{fig-hierarchical}) from $\mathcal{X}^{(t)}_u$.
\end{itemize}
After the iteration process, our algorithm provides a classifier as detailed in Section \ref{sec:analysis}. 
\begin{table}[h]

\centering
\begin{tabular}{|l|l|} 
 \hline
 $\mathcal{Y}_{h,i}^{(t)}$ & set of candidate labels for $\mathcal{X}_{h,i}$ at step $t.$  \\ 
 $c_{h,i}^{v}(y)$  &$v^{th}$ requested cost associated to the label 
   $y$ at the center $x_{h,i}$  \\ 
 $\mathcal{X}_{c}^{(t)}$  &current classified region at step $t$\\
 $\mathcal{X}_{u}^{(t)}$  &current unclassified region at step $t$\\
 $I_{t}(\mathcal{X}_{h,i})$ & classification uncertainty of the cell $\mathcal{X}_{h,i}$  at step $t$\\
 $U_{f(:;y)}^{(t)}(\mathcal{X}_{h,i})$ & upper confidence bound of the expected  \\
  & prediction cost $f(:,y)$ in $\mathcal{X}_{h,i}$ at step $t$\\
 $L_{f(:;y)}^{(t)}(\mathcal{X}_{h,i})$ & lower confidence bound of the expected prediction cost $f(:,y)$ in $\mathcal{X}_{h,i}$ at step $t$\\
 $A_r^{(t)}(\mathcal{X}_{h,i})$ & Boolean variable linked on the refinement of $\mathcal{X}_{h,i}$ at step $t$\\
 
 \hline
\end{tabular}
\caption{Main quantities that appears in the analysis of Algorithm \ref{alg:active cost-sensitive}}
\label{table:1}
\end{table}
\subsection{Analysis of our algorithm} 
\label{sec:analysis}
This Section is devoted to a brief analysis of our algorithm while the theoretical results are in Section \ref{sec:theo-results}. We provide some intuitions behind the construction of the set of candidate labels and describe the refinement and uncertainty criteria used in our algorithm. Our techniques are built upon the construction of upper and lower bounds of the expected prediction cost functions. We therefore first introduce the following definition:  
\begin{defi}
\label{def-conf}
Given a set $A\subset [0,1]^d$, and a deterministic function $g$ defined in $A$ that takes values in $[0,1]$, we say that $g$ has confidence bounds $L_g(A)$, $U_g(A)$ at level $\delta$ in the set $A$ if  with probability at least $1-\delta$, 
$$L_g(A)\leq g(x)\leq  U_g(A)\quad \text{for all}\; x \in A.$$
\end{defi}
As we will consider in Section \ref{sec:confidence bands}, the quantities $L_g(A)$, $U_g(A)$ are randomized and the probability $(1-\delta)$ is with respect  to their uncertainty. Importantly, the dependence of these confidence bounds with respect to $\delta$ is omitted for clarity purposes.

\textbf{How can we accurately construct the (random) set of candidate labels} $\bm{\mathcal{Y}^{(t)}_{h,i}?}$ For any cell $\mathcal{X}_{h,i}$, the corresponding set of candidate labels is initialized to $\mathcal{Y}_{h,i}^{(0)}=\mathcal{Y}$. For $t\geq 1$, let $\mathcal{X}_{h_t,i_t}$  be the cell chosen from the unclassified region by our algorithm at step $t.$ To compute 
$\mathcal{Y}^{(t)}_{h_t,i_t}$, we first need to provide confidence bounds on the expected prediction cost functions $f(.;y)$ for all $y\in \mathcal{Y}^{(t-1)}_{h_t,i_t}$. That is, according to Definition \ref{def-conf}, for a given parameter $\delta'\in (0,1)$, we have to provide the quantities $L_{f(:;y)}^{(t)}(\mathcal{X}_{h_t,i_t})$, $U_{f(:;y)}^{(t)}(\mathcal{X}_{h_t,i_t})$ such that with probability at least $1-\delta'$, we have for all $x\in\mathcal{X}_{h_t,i_t}$, 
\begin{equation}
   \label{eq-upper-lower-conf}
    L_{f(:;y)}^{(t)}(\mathcal{X}_{h_t,i_t})\leq f(x;y)\leq U_{f(:;y)}^{(t)}(\mathcal{X}_{h_t,i_t})
\end{equation}

Once confidence bounds are constructed, the set of candidate labels $\mathcal{Y}^{(t)}_{h,i}$ only contains labels $y$ from $\mathcal{Y}^{(t-1)}_{h,i}$ that satisfy: 
\begin{equation}
\label{eq:updateY}
L_{f(:;y)}^{(t)}(\mathcal{X}_{h_t,i_t})\leq\min_{y\in \mathcal{Y}^{(t-1)}_{h_t,i_t}}\, U_{f(:;y)}^{(t)}(\mathcal{X}_{h_t,i_t}),
\end{equation}

Particularly, under some smoothness assumption, for all $x$ $\in$ $\mathcal{X}_{h_t,i_t}$, $f^*_{cs}(x)$ is never discarded and $\mathcal{X}_{h_t,i_t}$ is therefore correctly labelled when it is added to the classified region $\mathcal{X}_c^{(t)}.$ Importantly, for any cell $\mathcal{X}_{h,i}$, the confidence bounds can be initialized to $L_{f(:;y)}^{(0)}(\mathcal{X}_{h,i})=-\infty$ and $U_{f(:;y)}^{(0)}(\mathcal{X}_{h,i})=+\infty$ and are progressively improved as costs are requested. These confidence bounds could take various forms depending on assumptions made on the considered problem of cost-sensitive classification. For instance in Section \ref{sec:confidence bands}, under some smoothness assumption on the expected prediction cost function $f(.;y)$, we will provide  specific expressions of $L_{f(:;y)}^{(t)}(\mathcal{X}_{h,i})$ and $U_{f(:;y)}^{(t)}(\mathcal{X}_{h,i})$  which can be used by our algorithm.\\

\textbf{How can we choose the most uncertain cell from the unclassified region?}
 To  measure the level of uncertainty of a cell $\mathcal{X}_{h,i}$ from the unclassified region $\mathcal{X}_u^{(t)}$, we can use the confidence bounds of each expected prediction cost function $f(.;y).$ For $\mathcal{X}_{h,i}$ from $\mathcal{X}_u^{(t)}$, let $I_t(x_{h,i})$ be the quantity defined as: 
 \begin{equation}
\label{eq:information}
I_t(\mathcal{X}_{h,i})=\min_{y\in \mathcal{Y}^{(t)}_{h,i}}\,U_{f(:;y)}^{(t)}(\mathcal{X}_{h,i})-\min_{y\in \mathcal{Y}^{(t)}_{h,i}} L_{f(:;y)}^{(t)}(\mathcal{X}_{h,i})
\end{equation}
According to Equation \eqref{eq:updateY}, the quantity $I_t(\mathcal{X}_{h,i})$ is non-negative (with high probability). The algorithm therefore  chooses to interact with  the cell $\mathcal{X}_{h,i}$  that has the largest value $I_t(\mathcal{X}_{h,i})$ over all the cells that belong to the unclassified region.

\textbf{Refinement criterion.} Once selecting the cell $\mathcal{X}_{h_t,i_t}$ from $\mathcal{X}_u^{(t)}$, the algorithm decides to refine when it becomes difficult to differentiate labels from the current set of candidate labels after a large number of interactions. In Algorithm \ref{alg:active cost-sensitive}, the refinement criteria is characterized by a Boolean variable $A^{(t)}_{r}(\mathcal{X}_{h_t,i_t})$ such that
\begin{equation}A_r^{(t)}(\mathcal{X}_{h_t,i_t})=\text{True}\;\Longleftrightarrow \;\mathcal{X}_{h_t,i_t}\;\text{has to be refined at step $t$}.
\end{equation}
For the refinement criteria, we could consider a \textit{cut-off} (possibly depending on the current depth $h_t$ and some complexity parameters) such that if the cell $\mathcal{X}_{h_t,i_t}$ has not been added to the classified region within a given number of interactions with the algorithm, the value $True$ is assigned to the variable $A^{(t)}_{r}(\mathcal{X}_{h_t,i_t})$ and thus the cell $\mathcal{X}_{h_t,i_t}$ has to be refined. In this case, as we pointed out in Section \ref{sec:algorithm description}, it is replaced by its children $\mathcal{X}_{h_t+1,2i_t},\mathcal{X}_{h_t+1,2i_t-1}$ in the unlabelled region $\mathcal{X}_u^{(t)}$. Moreover, the cells $\mathcal{X}_{h_t+1,2i_t}$,
$\mathcal{X}_{h_t+1,2i_t-1}$ inherit of some properties of $\mathcal{X}_{h_t,i_t}$ namely confidence bounds and the set of candidate labels.     

\textbf{Estimator of the Bayes cost-sensitive classifier.}
Let $T_n$ be the last step in Algorithm \ref{alg:active cost-sensitive}. The resulting active learning classifier provided by Algorithm \ref{alg:active cost-sensitive} is defined as follows:
    \begin{equation} 
    \label{eq:classifier}
    \hat{g}_n(x)=\argmin_{y\in\mathcal{Y}_{h,i}^{(T_n)}}\;U_{f(:;y)}^{(T_n)}(\mathcal{X}_{h,i}),
    \end{equation}
  for $x$ $\in$ $\mathcal{X}_{h,i}$, with $\mathcal{X}_{h,i} \in  \mathcal{X}^{(T_n)}_u\cup\mathcal{X}^{(T_n)}_c.$ As detailed in section \ref{sec:rate}, it enjoys goods statistical guarantees for cost-sensitive classification under some noise and smoothness assumptions.


\subsection{Theoretical results} 
\label{sec:theo-results}
This Section is devoted to theoretical guarantees achieved by our algorithm. We first begin by presenting theoretical properties of our algorithm under some assumptions and therefore the rate of convergence achieved by the resulting classifier $\hat{g}_n$ in Algorithm \ref{alg:active cost-sensitive}. Secondly, we complete the result on the rate of convergence by providing a (near) matching lower bound. For Clarity purpose, the proofs of these results are relegated to Section \ref{sec:appendix}.
\subsubsection{Assumptions}
\label{sec:assumption} 
In this Section, we consider three assumptions which are commonly used in the nonparametric setting for the analysis of the rates of convergence both in active and passive learning. 
\begin{ass}[Hölder-smoothness assumption]~\\
\label{ass:smoothness}
There exist $\alpha\leq 1$ and $L>0$ such that, for all $y$ $\in$ $\mathcal{Y}$, 
$$\vert f(x;y)-f(z;y)\vert \leq L\parallel x-z\parallel^{\alpha}\quad\text{for all}\;x,z\in\mathcal{X}.$$
\end{ass}
This assumption means that two close points tend to have the same prediction cost.  

\begin{ass}[Strong density assumption]
\label{ass:strong density}~\\The marginal probability on $\mathcal{X}$ admits a density $p_X$ and there exist $\mu_{min}, \mu_{max}>0$ such that for all $x$ with $p_X(x)>0$,
$$\mu_{min}\leq p_X(x)\leq \mu_{max}.$$
\end{ass}
The Assumption \ref{ass:strong density} is standard both in passive and active learning \citep{audibert2007fast,minsker2012plug,reeve2019learning}. However, it can be weakened by considering a particular hierarchical partition as stated in Section \ref{sec:model}. For instance, we can consider a partition of the space with cells (balls for example) in a hierarchical way, so that at depth $h>0$, for any cell $\mathcal{X}_{h,i}$ at level $h$, we have
$c_2 \rho^{hd} \leq P_X(\mathcal{X}_{h,i}) \leq c_1 \rho^{hd}$ (where $c_1$ and $c_2$ and $\rho$ are absolute constants and $P_X$ the marginal probability defined on $\mathcal{X})$
and such that the covering $($at depth $h)$ is somewhat tight, that is
$Supp(P_X) \subset \bigcup_i  \mathcal{X}_{h,i}$       and $\sum_i    P_X(\mathcal{X}_{h,i}) \leq C$
where $C$ is some universal constant.

Next, let us state the noise assumption that characterizes the behavior of the expected cost functions at the decision boundary. 

For $y\in \mathcal{Y}$, let $\Delta(x,y)$ be defined as:
\begin{equation}
    \label{eq-delta}
    \Delta(x,y)=f(x;y)-\min_{y'\in\mathcal{Y}}\,f(x;y').
\end{equation}
We define $\Delta(x)$ as 
\begin{equation}
\label{eq-delta1}
 \Delta(x)= \left\{
    \begin{array}{ll}
        \min_{y\in\mathcal{Y}}\lbrace \Delta(x,y):\;\Delta(x,y)>0\rbrace & \mbox{If there exists } y\in\mathcal{Y}\; \text{such that}\; \Delta(x,y)>0\\
        \infty & \mbox{otherwise.}
    \end{array}
\right. 
\end{equation}

Given an instance $x$, let $f(x;y^{(1)})\leq f(x;y^{(2)})\leq...\leq f(x;y^{(M)})$ denote order statistic on $f(x,y)$, $y\in\mathcal{Y}$. we define  $\Delta'(x)=f(x;y^{(2)})-f(x;y^{(1)}).$

\begin{ass}[Refined margin noise assumption]\label{Ass:margin-noise}~\\
There exist parameters $\beta, C_{\beta}, C'_{\beta}, \tau\geq 0$ such that for all $\epsilon>0$, 
\begin{equation} 
P_{X}(x\in  \mathcal{X},\;\;\Delta(x)\leq \epsilon)\leq  C_{\beta}\epsilon^{\beta},\quad\text{and}\quad P_{X}(x\in \mathcal{X},\;\;\Delta'(x)\leq \epsilon)\leq \tau+ C'_{\beta}\epsilon^{\beta}
\end{equation}
\end{ass}

An equivalent version of Assumption \ref{Ass:margin-noise} was recently introduced in the setting of classical active learning \citep{kpotufe2022nuances}. Assumption \ref{Ass:margin-noise} generalizes the Tsybakov's noise assumption used in previous works on cost-sensitive active learning, especially in parametric setting \citep{krishnamurthy2017active,JMLR:v20:17-681}. The probability mass of the region where the Bayes cost-sensitive classifier is not unique is taken into account by $\tau$. Particularly, when $\tau=0$, we recover the assumption used in \citep{JMLR:v20:17-681} which assumes the uniqueness of the Bayes cost-sensitive classifier.

\subsubsection{Specific choice of confidence bounds}
\label{sec:confidence bands}
In this Section, we will provide  precise expressions of confidence bounds \eqref{eq-upper-lower-conf} on the expected prediction cost function under Assumption \ref{ass:smoothness}. 
Before, let us introduce the following quantities:  
\begin{equation}
    \label{eq:biais00}
B_h=\left(\nu_2\rho^h\right)^{\alpha},\quad V(n_a)=\sqrt{\frac{\log(2n^3M)}{2n_a}}
\end{equation}
where $\nu_2, \rho$ come from \eqref{eq:ball-ass1}, $\alpha$ comes from Assumption \ref{ass:smoothness}, $n$ is the label budget, $M$ the number of labels, and $n_a\leq n$ an integer. Let $y$ $\in$ $\mathcal{Y}^{(t)}_{h,i}$, and $c_{h,i}^{v}(y)$ the $v$-th requested prediction cost associated to $y$ at $x_{h,i}.$ For all $y$ $\in$ $\mathcal{Y}^{(t)}_{h,i}$, we consider the following estimator of $f(x_{h,i};y)$ at step $t$:
\begin{equation}
\label{eq:estimator}
\hat{f}^{(t)}(x_{h,i};y)=\frac{1}{n_{h,i}(t)}\sum_{v=1}^{n_{h,i}(t)}c_{h,i}^{v}(y),
\end{equation}
where $n_{h,i}(t)$ is the number of times the algorithm has interacted with the cell $\mathcal{X}_{h,i}$ up to step $t$.

Additionally, for all $x$ $\in$ $\mathcal{X}_{h,i}$, for all $y$ $\in$  $\mathcal{Y}^{(t)}_{h,i}$, we define the estimator $\hat{f}^{(t)}(x;y)$ of $f(x;y)$ at step $t$ as: 
$$\hat{f}^{(t)}(x;y):=\hat{f}^{(t)}(x_{h,i};y).$$
The expression of the confidence bounds are built upon the following decomposition for all $x\in\mathcal{X}_{h,i}$ 

$$\vert\hat{f}^{(t)}(x;y)-f(x;y)\vert\leq  \vert\hat{f}^{(t)}(x_{h,i};y)-f(x_{h,i};y)\vert +\vert f(x;y)-f(x_{h,i};y)\vert$$
Let us assume that there exists a favorable event $E$ (which will be explicitly clarified in Section \ref{sec:appendix}), in which  for any step $t$, for any cell $\mathcal{X}_{h,i}$ that has interacted with the algorithm, we have for all $y$ $\in$  $\mathcal{Y}^{(t)}_{h,i}$: 

\begin{equation}
    \label{eq:hoefdding01}
     \vert\hat{f}^{(t)}(x_{h,i};y)-f(x_{h,i};y)\vert\leq V(n_{h,i}(t))
\end{equation}
In this case, by using Assumption \ref{ass:smoothness}, and Equation \eqref{eq:ball-ass1},      we have for any $x$ $\in$ $\mathcal{X}_{h,i}$, for all $y$ $\in$ $\mathcal{Y}_{h,i}^{(t)}$: 
\begin{equation}
\label{eq:lower00}
    f(x;y)\geq\bar{L}_{f(:;y)}^{(t)}(\mathcal{X}_{h,i}):=\hat{f}^{(t)}(x_{h,i};y)-V(n_{h,i}(t))-B_h,
\end{equation}
\begin{equation}
\label{eq:upper00}
    f(x;y)\leq\bar{U}_{f(:;y)}^{(t)}(\mathcal{X}_{h,i}):=\hat{f}^{(t)}(x_{h,i};y)+V(n_{h,i}(t))+B_h.
\end{equation}
The confidence bounds are thus defined as:
\begin{equation}
\label{eq:lower01}
    L_{f(:;y)}^{(t)}(\mathcal{X}_{h,i})=\max\left(\bar{L}_{f(:;y)}^{(t)}(\mathcal{X}_{h,i}),L_{f(:;y)}^{(t-1)}(\mathcal{X}_{h,i})\right),
\end{equation}
\begin{equation}
\label{eq:upper01}
    U_{f(:;y)}^{(t)}(\mathcal{X}_{h,i})=\min\left(\bar{U}_{f(:;y)}^{(t)}(\mathcal{X}_{h,i}),U_{f(:;y)}^{(t-1)}(\mathcal{X}_{h,i})\right).
\end{equation}
with $L_{f(:;y)}^{(0)}(\mathcal{X}_{h,i})=-\infty, \;U_{f(:;y)}^{(0)}(\mathcal{X}_{h,i})=+\infty$ for all $h,i.$

\subsubsection{Rate of convergence}
\label{sec:rate}
Before providing the result on the rate of convergence, let us introduce the following definition:
\begin{defi}[Cost-sensitive classification measure]\label{def:cost-sensitive measure}~\\
Let $\Xi=(0,1)\times(0,+\infty)^{4}\times[0,1]\times (0,1)\times(0,\infty)$. For any $\zeta=(\alpha,L,\beta,C_{\beta},C'_{\beta},\tau,\mu_{\min},\mu_{\max})\in \Xi$, we denote $\mathcal{P}_{cs}(\zeta)$ the class of probability measures $P$ on $\mathcal{X}\times [0,1]^M$ such that: (a) The expected conditional cost functions $($associated to $P)$ satisfy Assumption~\ref{ass:smoothness} with parameter $\alpha$, and $L$. (b) The probability $P$ satisfies Assumption \ref{Ass:margin-noise} with parameters $\beta, C_{\beta}, C'_{\beta}, \tau$. (c) The probability $P$ satisfies Assumption \ref{ass:strong density} with parameters $\mu_{\min},\mu_{\max}$. 

\end{defi}
\begin{algorithm}[!htbp]
\caption{Nonparametric active learning algorithm for cost-sensitive classification}
\label{alg:active cost-sensitive}
\KwIn{Budget $n$}  
\textbf{Initialization} $\ell=0$ (current budget), $t=0$, $\mathcal{X}_{0,1}=[0,1]^d$\\
$\mathcal{X}_u^{(t)}=\{\mathcal{X}_{0,1}\}$ \quad\quad //The unclassified region\\
$\mathcal{X}_c^{(t)}=\{\}$ \quad\quad //The classified region\\
$\mathcal{Y}_{h,i}^{(t)}=[M]$ the candidate labels for $\mathcal{X}_{h,i}$ $\in$ $\mathcal{X}_u^{(t)}$\\
For all $\mathcal{X}_{h,i}$ $\in$ $\mathcal{X}_u^{(t)}$, for all $y$ $\in$ $\mathcal{Y}_{h,i}^{(t)}$,   $L_{f(:;y)}^{(t)}(\mathcal{X}_{h,i})=-\infty$ and $U_{f(:;y)}^{(t)}(\mathcal{X}_{h,i})=+\infty$\\
$A_r^{(t)}(\mathcal{X}_{h,i})=False$ for all $\mathcal{X}_{h,i}$ $\in$ $\mathcal{X}_u^{(t)}$\\
\While{$\ell\leq n$}{//Choose a candidate cell with most uncertainty according to \eqref{eq:information}\\ $\mathcal{X}_{h_t,i_t} \in \argmax_{\mathcal{X}_{h,i}\in\mathcal{X}_u^{(t)}}\; I_t(\mathcal{X}_{h,i})$ \\
\If{$A_r^{(t)}(\mathcal{X}_{h_t,i_t})=True$}{//Refine and pass information to the next depth\\$\mathcal{X}_u^{(t)}=\mathcal{X}_u^{(t)}\setminus \lbrace \mathcal{X}_{h_t,i_t}\rbrace\cup \lbrace \mathcal{X}_{h_t+1,2i_t-1}; \mathcal{X}_{h_t+1,2i_t}\rbrace$\\
$\mathcal{Y}_{h_t+1,2i_t-1}^{(t)}=\mathcal{Y}_{h_t,i_t}^{(t)}$\\
$\mathcal{Y}_{h_t+1,2i_t}^{(t)}=\mathcal{Y}_{h_t,i_t}^{(t)}$\\
$L_{f(:;y)}^{(t)}(\mathcal{X}_{h_t+1,2i_t-1})=L_{f(:;y)}^{(t)}(\mathcal{X}_{h_t,i_t})$\\ $U_{f(:;y)}^{(t)}(\mathcal{X}_{h_t+1,2i_t-1})=U_{f(:;y)}^{(t)}(\mathcal{X}_{h_t,i_t})$\\
$L_{f(:;y)}^{(t)}(\mathcal{X}_{h_t+1,2i_t})=L_{f(:;y)}^{(t)}(\mathcal{X}_{h_t,i_t})$\\ $U_{f(:;y)}^{(t)}(\mathcal{X}_{h_t+1,2i_t})=U_{f(:;y)}^{(t)}(\mathcal{X}_{h_t,i_t})$\\
$A_r^{(t)}(\mathcal{X}_{h_t+1,2i_t})=False$\\
$A_r^{(t)}(\mathcal{X}_{h_t+1,2i_t-1})=False$
}  
\Else{//Interact with  $\mathcal{X}_{h_t,i_t}$ via its center $x_{h_t,i_t}$\\
Query cost of predicting $y$ for any $y$ $\in \mathcal{Y}_{h_t,i_t}^{(t)}$ at $x_{h,i}$\\
$\ell=\ell+1$\\
Update $A_r^{(t)}(\mathcal{X}_{h_t,i_t})$\\
\For{$y\in \mathcal{Y}_{h_t,i_t}^{(t)}$}{
Update the confidence bounds $L_{f(:;y)}^{(t)}(\mathcal{X}_{h_t,i_t})$ and $U_{f(:;y)}^{(t)}(\mathcal{X}_{h_t,i_t})$.}
$\mathcal{Y}_{h_t,i_t}^{(t)}=\lbrace y\in \mathcal{Y}^{(t)}_{h_t,i_t},\;L_{f(:;y)}^{(t)}(\mathcal{X}_{h_t,i_t})\leq\min_{y\in \mathcal{Y}^{(t)}_{h_t,i_t}}\,U_{f(:;y)}^{(t)}(\mathcal{X}_{h_t,i_t})\rbrace$\\
\If{$\vert \mathcal{Y}_{h_t,i_t}^{(t)}\vert=1$}{$\mathcal{X}_c^{(t)}=\mathcal{X}_c^{(t)}\cup \lbrace \mathcal{X}_{h_t,i_t}\rbrace$\\
$\mathcal{X}_u^{(t)}=\mathcal{X}_u^{(t)}\setminus \lbrace \mathcal{X}_{h_t,i_t}\rbrace$}
}$t=t+1$}
\KwOut{Classifier $\hat{g}_n$ defined by \eqref{eq:classifier}}
\end{algorithm}
\begin{theo}[Rate of convergence]\label{theo:upper-bounds}~\\
For any $\zeta=(\alpha,L,\beta,C_{\beta},C'_{\beta},\tau,\mu_{\min},\mu_{\max})\in \Xi$ (introduced in Definition \ref{def:cost-sensitive measure}) with $\alpha\beta\leq d$, there exists an absolute constant $c_{cs}$ such that for all $\delta$ $\in$ $(0,\frac{1}{2})$, for any $P$ $\in$ $\mathcal{P}_{cs}(\zeta)$, and label budget $n\geq \frac{1}{\delta}$, if at each step $t$ of Algorithm \ref{alg:active cost-sensitive}, we consider: 
\begin{itemize}
    \item $A_r^{(t)}(\mathcal{X}_{h_t,i_t})=True$ is equivalent to $V(n_{h_t,i_t}(t))\leq 2B_{h_t}$, where $n_{h,i}(t)$ is the number of interactions made with the cell $\mathcal{X}_{h,i}$ up to step $t$, $V(.)$ and $B_h$ are defined in \eqref{eq:biais00}. 
    \item The confidence bounds are updated according to \eqref{eq:lower01} and \eqref{eq:upper01}.
\end{itemize}
Then, the classifier provided by Algorithm \ref{alg:active cost-sensitive} satisfies with probability at least $1-\delta$, 
\begin{equation}
\varepsilon_{cs}(\hat{g}_n)\leq\left\{
 \begin{array}{ll}   
        \left(\frac{\tau}{n}.c_{cs}\log\left(2n^3M\right)\right)^{\frac{\alpha(\beta+1)}{2\alpha+d}} \quad\text{if}\; \tau\geq \tau_{0}\\
        \left(\frac{1}{n}.c_{cs}\log\left(2n^3M\right)\right)^{\frac{\alpha(\beta+1)}{2\alpha+d-\alpha\beta}}\;\text{if}\;\tau\leq \tau_0,  
    \end{array}
\right.
\end{equation}

where $\tau_0=\left(\frac{1}{n}.c_{cs}\log\left(2n^3M\right)\right)^{\frac{\alpha\beta}{2\alpha+d-\alpha\beta}}.$

\end{theo}

Our result shows that improvement over passive learning \citep{reeve2019learning} depends on the probability mass $\tau$ of region where the Bayes cost-sensitive classifier is non unique. When $\tau$ is small enough, the obtained rate is better than the classical rate obtained in passive learning \citep{reeve2019learning} which is of order of $n^{-\frac{\alpha(\beta+1)}{2\alpha+d}}$. On the other hand, when $\tau$ is large enough, the obtained rate has the same order as in the passive learning counterpart \citep{reeve2019learning}.  

\subsubsection{Lower bounds}
In this Section, we provide a result that shows the optimality of the rate of convergence provided in Theorem \ref{theo:upper-bounds}. The result is stated for the binary case $(M=2)$. However, the extension to multi-class is not difficult as we can always consider the subset of probabilities $P\in\mathcal{P}_{cs}(\zeta)$ with the expected prediction cost functions $f(.;3)\equiv\ldots f(.;M)\equiv 1$ and work only with $f(.;1) ,\; f(.;2)$ to obtain the corresponding lower bounds.  

\begin{theo}[Lower bound for cost-sensitive classification]\label{theo-lower}~\\
Let us consider the cost-sensitive classification problem with $M=2$ and $\mathcal{Y}=\{0,1\}$. Let $\zeta=(\alpha,L,\beta,C_{\beta},C'_{\beta},\tau,\mu_{\min},\mu_{\max})\in \Xi$ introduced in Definition \ref{def:cost-sensitive measure}. We assume that $L,C_{\beta},C'_{\beta} \in (1,\infty)$,  $\alpha\beta \leq d$ and $\mu_{min}\in (0,\tau)$. There exist  constants $C_{cs}$, $C_{cs}'$  (independent of $n$) such that:
$$\inf_{\hat{g}_n}\sup_{P\in\mathcal{P}_{cs}(\zeta)} \mathbb{E} (R_{cs}(\hat{g}_n))-R_{cs}(f^*_{cs}))\geq  \min(a_{n,\tau},a'_{n,\tau}),$$

where 
$$a_{n,\tau}=C_{cs}\max\left(\left(\frac{\tau}{n}\right)^{\frac{\alpha(\beta+1)}{2\alpha+d}}, \left(\frac{1}{n}\right)^{\frac{\alpha(\beta+1)}{2\alpha+d-\alpha\beta}}\right),\quad a'_{n,\tau}=C'_{cs}
\max\left(\left(\frac{1}{n}\right)^{\frac{\alpha(\beta+1)}{2\alpha+d}}, \left(\frac{1}{\tau}\right)^{\frac{\alpha(\beta+1)}{2\alpha+d}}\left(\frac{1}{n}\right)^{\frac{\alpha(\beta+1)}{2\alpha+d-\alpha\beta}}\right),$$
the infimum is taken over all active learning algorithms that provide a classifier $\hat{g}_{n}:\mathcal{X}\longrightarrow \{0,1\}$ based on sample (strategy) $S_n=\{(X_1,c_1),\ldots,(x_n,c_n)\}$ and the supremum runs over
all probability $P$ $\in$ $\mathcal{P}_{cs}(\zeta).$ 
\end{theo}

The above Theorem \ref{theo-lower} provides a lower bound that matches (up to a logarithmic factor) the upper bound  provided in Theorem \ref{theo:upper-bounds}. Moreover, we can get an improvement over the lower bound obtained in the passive learning counterpart \citep{reeve2019learning}. Indeed, our result states that an active learner can outperform the passive counterpart when $\tau$ is small. In this case, the minimax rate is of order of $n^{-\frac{\alpha(\beta+1)}{2\alpha+d-\alpha\beta}}$, whereas in passive learning, it is of order of $n^{-\frac{\alpha(\beta+1)}{2\alpha+d}}.$ On the other hand, when the probability-mass $\tau$ is large enough, no active learner can outperform the passive counterpart and the minimax rate remains the same as in the passive learning counterpart.    

\section{Conclusion and perspectives}
\label{sec:conclusion}
In this paper, we studied nonparametric active learning for cost-sensitive multi-class classification. We designed an active learning algorithm that provided a classifier for cost-sensitive multi-class classification. Under a general noise assumption, we proved that our algorithm achieves optimal rate of convergence, and the gain over the corresponding passive learning is explicitly determined by the probability-mass of the boundary decision. An interesting future direction is to evaluate the gain over passive learning under our general noise assumption, when considering the parametric setting as in \citep{JMLR:v20:17-681}.


\newpage
\appendix
\section{Missing proofs}
\label{sec:appendix}
This Section is organized as follows:   in Section~\ref{sec:notations_A}, we introduce some additional notations. In Section~\ref{sec:prooftheoupp}  we formally prove Theorem~\ref{theo:upper-bounds} from the main manuscript, and in Section \ref{sec:prooftheolower} we also formally prove Theorem \ref{theo-lower}.

\subsection{Notations}
\label{sec:notations_A} 
Let $T_n$ be the last step in Algorithm \ref{alg:active cost-sensitive}, that is when the label budget is exhausted. Let $\mathcal{X}^{(t)}$ be the set of cells with which the algorithm has interacted at least once up to the end of a given step $t$, and $\tilde{\mathcal{X}}=\cup_{t\leq T_n} \mathcal{X}^{(t)} $ the set of all cells with which the algorithm has interacted along the learning process.
For a fixed $t$, and $\mathcal{X}_{h,i}$ $\in$ $\mathcal{X}^{(t)}$, for all $y$ $\in$ $\mathcal{Y}_{h,i}^{(t)}$, let $\hat{f}^{(t)}(x_{h,i};y)$ be the empirical estimate of $f(x_{h,i};y)$ as defined by:
\begin{equation}
\label{eq:estimator-suppl}
\hat{f}^{(t)}(x_{h,i},y)=\frac{1}{n_{h,i}(t)}\sum_{v=1}^{n_{h,i}(t)}c_{h,i}^{v}(y),
\end{equation} 
where $n_{h,i}(t)$ is the number of times the algorithm has interacted with the cell $\mathcal{X}_{h,i}$ up to the end of step $t$, and $c_{h,i}^{v}(y)$ the $v$-th requested prediction cost associated to $y$ at $x_{h,i}.$ Let us introduce the following parameters

\begin{equation}
    \label{eq:biais00supp}
B_h=\left(\nu_2\rho^h\right)^{\alpha},\quad V(n_a)=\sqrt{\frac{\log(2n^3M)}{2n_a}}
\end{equation}
where $\nu_2, \rho$ come from Section \ref{sec:model}, $\alpha$ comes from Assumption \ref{ass:smoothness}, $n$ is the label budget, $M$ the number of labels, and $n_a\leq n$ an integer.
\subsection{Proof of Theorem \ref{theo:upper-bounds}}
\label{sec:prooftheoupp}
Our proof is inspired by \citep{munos2014bandits,shekhar2021active,JMLR:v20:17-681}, and built on the definition of the following event defined as $E=\cap_{t\geq 0} E_t$ where:

$$E_t=\left\lbrace \forall\; \mathcal{X}_{h,i}\in \mathcal{X}^{(t)}, \forall \,y\in \mathcal{Y}_{h,i}^{(t)}, \; \vert \hat{f}^{(t)}(x_{h,i};y)-f(x_{h,i},y)\vert \leq V(n_{h,i}(t))\right\rbrace$$
In the following Lemma, we prove that the event $E$ happens with high probability when $n$ is large enough.
\begin{lem}\label{lem-favorable event}~\\
 We have: 
 $P(E)\geq 1-\tfrac{1}{n}$
\end{lem}
\begin{proof}~\\
 Let $N=\vert \tilde{\mathcal{X}}\vert$ (where $\tilde{\mathcal{X}}$ is defined in Section \ref{sec:notations_A}) and $T_n$ the last step in our algorithm. For $\mathcal{X}_{h_i,j_i} \in \tilde{\mathcal{X}}$,  let $n_{h_i,j_i}(T_n)$ the number of times the algorithm has interacted with $\mathcal{X}_{h_i,j_i}$ up to step $T_n$ and for $u\leq n_{h_i,j_i}(T_n)$ let $\mathcal{Y}_{{h_i,j_i}}^{(u)}$ the remaining set of  labels at the $u$-th interaction. Given $y\in \mathcal{Y}_{h_i,j_i}^{(u)}$, we note $c_{h_i,j_i}^{u}(y)$ the associated cost. Clearly, we have: 
$$E=\left\lbrace \forall\; 1\leq i\leq N,\;\forall\, 1\leq u\leq n_{h_i,j_i}(T_n),\;\forall y\in \mathcal{Y}_{h_i,j_i}^{(u)},\; \left\vert\frac{1}{u}\sum_{v=1}^{u}c_{h_i,j_i}^{v}(y)-f(x_{h_i,j_i};y)\right\vert\leq\sqrt{\frac{\log(2n^3M)}{2u}}\right\rbrace$$

Given $\mathcal{X}_{h_i,j_i}$, we have $\mathbb{E}(c_{h_i,j_i}^{v}(y))=f(x_{h_i,j_i};y)$ for all $u\leq n_{h_i,j_i}(T_n)$, $v\leq u$ and $y\in \mathcal{Y}_{h_i,j_i}^{(u)}$. Then by Hoeffding's inequality, we have: 

\begin{align*}
    \mathbb{P}(E^c)&\leq \mathbb{E}\left(\sum_{i\leq N}\sum_{u\leq n_{h_i,j_i}(T)}\sum_{y\in \mathcal{Y}_{h_i,j_i}^{(u)}}\mathbb{P}\left(\left\vert\frac{1}{u}\sum_{v=1}^{u}c_{h_i,j_i}^{v}(y)-f(x_{h_i,j_i};y)\right\vert\geq \sqrt{\frac{\log(2n^3M)}{2u}}\mid \tilde{\mathcal{X}}\right)\right)\\
    &\leq \mathbb{E}\left(\sum_{i\leq N}\sum_{u\leq n_{h_i,j_i}(T)}\sum_{y\in \mathcal{Y}_{h_i,j_i}^{(u)}} 2\exp(\log(\frac{1}{2Mn^3}))\right)\\
    &\leq \frac{1}{n}\quad \text{as}\; N,n_{h_i,j_i}(T)\leq n, \text{and}\;\vert \mathcal{Y}_{h_i,j_i}^{(u)}\vert\leq M,
\end{align*}
which leads to $\mathbb{P}(E)\geq 1-\frac{1}{n}$.
\end{proof}

\begin{lem}\label{lemma:correct}~\\
Let us assume that Assumption \ref{ass:smoothness} holds. On the event $E$, suppose that the Algorithm \ref{alg:active cost-sensitive} has progressed up to step $t\geq 1$ and that the update of upper and lower confidence bounds are made according to Equations \eqref{eq:upper01}, \eqref{eq:lower01} in he main text. Let $\mathcal{X}_{h_t,i_t}$ be the chosen cell from $\mathcal{X}^{(t)}_u$. For all $x$ $\in$ $\mathcal{X}_{h_t,i_t}$, the labels $y'$ $\in$ $\argmin_{y\in \mathcal{Y}}\;f(x,y)$ are never eliminated from $\mathcal{Y}_{h_t,i_t}^{(t)}$. Consequently, the cell $\mathcal{X}_{h_t,i_t}$ is correctly labelled it is added to $\mathcal{X}^{(t+1)}_c$.
\end{lem}

\begin{proof}~\\
Let $y$ $\in$ $\mathcal{Y}_{h_t,i_t}^{(t)}$ and $x$ $\in$ $\mathcal{X}_{h_t,i_t}$. By Assumption \ref{ass:smoothness}, we have: 
$$\vert f(x;y)-f(x_{h_t,i_t};y)\vert \leq B_{h_t}.$$

By the definition of the event $E$, we have 
$$\vert f(x_{h_t,i_t};y)-\hat{f}^{(t)}(x_{h_t,i_t};y)\vert \leq V(n_{h_t,i_t}),$$ 
and thus, 
\begin{align*}
\vert f(x;y)-\hat{f}^{(t)}(x;y)\vert &=\vert f(x;y)-\hat{f}^{(t)}(x_{h_t,i_t};y)\vert\\
                               &\leq V(n_{h_t,i_t})+ B_{h_t}.
\end{align*}
Consequently, 
\begin{equation}
\label{eq:lem-noerror1}
\hat{f}(x_{h_t,i_t};y)-V(n_{h_t,i_t})- B_{h_t}\leq f(x;y)\leq \hat{f}(x_{h_t,i_t};y)+V(n_{h_t,i_t})+ B_{h_t}
\end{equation}
By applying \eqref{eq:lem-noerror1} at time $t-1$ $($ that could be either at depth $h_t-1$, or $h_t)$, we  obtain:

\begin{equation}
\label{eq:boundconfi}
L_{f(:;y)}^{(t)}(\mathcal{X}_{h,i})\leq f(x,y) \leq U_{f(:;y)}^{(t)}(\mathcal{X}_{h,i}),
\end{equation}

where $L_{f(:;y)}^{(t)}(\mathcal{X}_{h,i})$ and $U_{f(:;y)}^{(t)}(\mathcal{X}_{h,i})$ are defined by \eqref{eq:lower01} and \eqref{eq:upper01} in the main text.\\
Let $y'$ $\in$ $\argmin_{y\in \mathcal{Y}}\;f(x;y)$. If $y'$ $\in$ $\mathcal{Y}_{h_t,i_t}^{(t)}$ and $y'$ $\notin$ $\mathcal{Y}_{h_t,i_t}^{(t+1)}$, then 
\begin{align*}
f(x;y')\geq L_{f(:;y')}^{(t)}(\mathcal{X}_{h_t,i_t})&>\min_{y\in \mathcal{Y}_{h_t,i_t}^{(t)}}\,U_{f(:;y)}^{(t)}(\mathcal{X}_{h_t,i_t})\\
                                    &=U_{f(:;\hat{y})}^{(t)}(\mathcal{X}_{h_t,i_t})\quad \text{where}\; \hat{y}\in\argmin_{y\in \mathcal{Y}_{h_t,i_t}^{(t)}}\,U_{f(:;\hat{y})}^{(t)}(\mathcal{X}_{h_t,i_t})\\
                                    &\geq f(x;\hat{y})
\end{align*}
which contradicts the fact that $y'\in \argmin_{y\in \mathcal{Y}}\,f(x;y).$ Then $y'$ is never eliminated and then, the cell $\mathcal{X}_{h_t,i_t}$ is correctly labelled if it is added to $\mathcal{X}^{(t+1)}_c$. 

\end{proof}

\begin{lem}\label{Lemma-expand}~\\
Let $h\geq 0$, and some $i\in \lbrace 1,\ldots,2^h\rbrace$. Let us assume that the assumption \ref{ass:smoothness} holds. For a fixed step $t\geq 1$, let us assume that for some $\mathcal{X}_{h,i}$ $\in$ $\mathcal{X}_u^{(t)}$

\begin{itemize}
    \item We have $\Delta(x)>12B_h$ (where $\Delta(.)$ is defined by \ref{eq-delta}) for some $x$ $\in$ $\mathcal{X}_{h,i}$.
    \item In Algorithm \ref{alg:active cost-sensitive}, the update of upper and lower confidence bounds are made according to Equations \eqref{eq:upper01}, \eqref{eq:lower01} which are stated in the main text.
    \item In Algorithm \ref{alg:active cost-sensitive}, the refinement condition $A_r^{(t)}(\mathcal{X}_{h,i})=True$ is equivalent to $V(n_{h,i}(t))\leq 2B_{h}$, where $n_{h,i}(t)$ is the number of interactions made with the cell $\mathcal{X}_{h,i}$ up to step $t$.
\end{itemize} Then, on the event $E$: either the cell $\mathcal{X}_{h,i}$ will never be refined, or the cell will be refined at some step $t$ and the  remaining set of candidate labels $\mathcal{Y}_{h,i}^{(t)}$ at that step only contains $y'\in\argmin_{y\in\mathcal{Y}}\,f(x;y)$.   
\end{lem}

\begin{proof}~\\
Let $x$ $\in$ $\mathcal{X}_{h,i}$ with $\Delta(x)>12B_h$. If $\Delta(x)=+\infty$, the Lemma is proven. Now, let us assume $\Delta(x)<\infty$, and let $\bar{y}\in \mathcal{Y}$ with $\Delta(x,\bar{y})>12B_h$. The Lemma is completely proven if we prove that $\bar{y}$ will be eliminated before a possible refinement.\\ Let $y'\in\argmin_{y\in\mathcal{Y}}\,f(x;y)$. By Assumption  \eqref{ass:smoothness}, 
\begin{equation}
\label{eq:lemma-expand1}
f(x;\bar{y})-f(x;y')\leq 2B_h+f(x_{h,i};\bar{y})-f(x_{h,i};y').
\end{equation}
If the cell $\mathcal{X}_{h,i}$ is refined at some step $1\leq t\leq T_n$ (where $T_n$ is the step at which the budget is reached), then  $\vert\mathcal{Y}_{h,i}^{(t)}\vert\geq 2$. By Lemma\ref{lemma:correct}, we have $y'$ $\in$ $\mathcal{Y}_{h,i}^{(t)}$. Let $\hat{y}$ be defined as:  
$$\hat{y}\in\argmin_{y\in\mathcal{Y}_{h,i}^{(t)}}\, \hat{f}^{(t)}(x_{h,i};y).$$
If $\bar{y}\in \mathcal{Y}_{h,i}^{(t)}$,  we have on the event $E$, and by  \eqref{eq:lemma-expand1}: 
\begin{align*}
  f(x;\bar{y})-f(x;y')&\leq 2V(n_{h,i}(t))+2B_h+\hat{f}^{(t)}(x_{h,i},\bar{y})-\hat{f}^{(t)}(x_{h,i},\hat{y})  
\end{align*}
Then, as $12B_h<f(x;\bar{y})-f(x;y')$, we get 
$$12B_h<2V(n_{h,i}(t))+2B_h+\hat{f}^{(t)}(x_{h,i};\bar{y})-\hat{f}^{(t)}(x_{h,i};\hat{y})$$
and thus by using the refinement rule $(V(n_{h,i}(t))\leq 2B_h)$, we obtain: 

$$\hat{f}^{(t)}(x_{h,i};\hat{y})+B_h+V(n_{h,i}(t))<\hat{f}^{(t)}(x_{h,i};\bar{y})-B_h-V(n_{h,i}(t))$$
and 
$$U_{f(:;\hat{y})}^{(t)}(\mathcal{X}_{h,i})<L_{f(:;\bar{y})}^{(t)}(\mathcal{X}_{h,i}),$$
which contradicts the fact that $\bar{y}$ $\in$ $\mathcal{Y}_{h,i}^{(t)}$ (see line 28 in algorithm \ref{alg:active cost-sensitive}).  
\end{proof}

\begin{defi}[Active cell]\label{def:active}~\\
Let $h>0$ an integer, and a cell $\mathcal{X}_{h,i}$ $(i=1,\ldots,2^h)$. Let $T_n$ be the last step at which the algorithm \ref{alg:active cost-sensitive} stops. We define $$\hat{t}(x_{h,i})=\min\;\{1\leq t\leq T_n,\;\; \mathcal{X}_{h,i} \in \mathcal{X}_u^{(t)}\}$$
By convention, $\hat{t}(x_{h,i})=+\infty$ if $\mathcal{X}_{h,i}\notin \mathcal{X}_u^{(t)}$ for all $1\leq t\leq T_n.$

The cell $\mathcal{X}_{h,i}$ is said to be \textit{active} if  $\hat{t}(x_{h,i})<+\infty.$
\end{defi}
\begin{lem}{(Number of active cells at depth h)}\label{lemma-activeset}~\\
Let $T_n$ be the last step at which Algorithm \ref{alg:active cost-sensitive} stops. Let $h\geq 1$ and  $N_{act,h}(T_n)$ be the number of active cells at depth $h$. Let us assume that Assumptions \ref{ass:smoothness}, \ref{Ass:margin-noise}, \eqref{ass:strong density} hold and that: 
\begin{itemize}
\item In Algorithm \ref{alg:active cost-sensitive}, the update of upper and lower confidence bounds are made according to Equations \eqref{eq:upper01}, \eqref{eq:lower01} which are stated in the main text.
\item In Algorithm \ref{alg:active cost-sensitive}, for all $1\leq t\leq T_n$, the refinement condition $A_r^{(t)}(\mathcal{X}_{h,i})=True$ is equivalent to $V(n_{h,i}(t))\leq 2B_{h}$, where $n_{h,i}(t)$ is the number of interactions made with the cell $\mathcal{X}_{h,i}$ up to step $t$.
\end{itemize}
Then on the event $E$, we have: 

\begin{equation}
\label{eq:lemma-activeset1}
    N_{act,h}(T_n)\leq C\frac{\tau}{\rho^{hd}} +C'\rho^{h(\alpha\beta-d)},
\end{equation}
where $C$ and $C'$ are absolute constants.
\end{lem}

\begin{proof}~\\
Let $h\geq 1$ be an integer, and $\tilde{\mathcal{X}}_h$ be the set of active cells at depth $h$ and at step $T_n$. We have that, on the event $E$ and by Assumption  \ref{ass:strong density}: 

\begin{equation} 
\label{eq:lemma-activeset2}
P_X\left(\bigcup_{\mathcal{X}_{h,i}\in\tilde{\mathcal{X}}_h}\mathcal{X}_{h,i}\right)=\sum_{i=1}^{N_{act,h}(T_n)} P_X\left(\mathcal{X}_{h,i}\right)\geq C'_0N_{act,h}(T_n) \nu_1^d\rho^{hd},
\end{equation}
where $C'_0$ is an absolute constant which only depends on $d$. 

Besides, by Lemma \ref{Lemma-expand} and Assumption \ref{Ass:margin-noise}, we have: 
$$P_X\left(\bigcup_{\mathcal{X}_{h,i}\in\tilde{\mathcal{X}}_h}\mathcal{X}_{h,i}\right)\leq \tau+C_0 \rho^{h\alpha\beta},$$
where $C_0=\frac{C_{\beta}12^{\beta}}{\rho^{\alpha\beta}}\nu_{2}^{\alpha\beta}.$
Consequently, combining with \eqref{eq:lemma-activeset2}, we get
$$N_{act,h}(T_n)\leq C\frac{\tau}{\rho^{hd}} +C'\rho^{h(\alpha\beta-d)}.$$

\end{proof}

\begin{lem}{(Largest depth $h_{max})$}\label{lemma-largestdepht}~\\
Let $T_n$ the step at which the budget has been reached. Let $h_{max}(n)$  be the largest depth at that step. Let us assume that Assumptions \ref{ass:smoothness}, \ref{Ass:margin-noise}, \eqref{ass:strong density} hold and that: 
\begin{itemize}
\item In Algorithm \ref{alg:active cost-sensitive}, the update of upper and lower confidence bounds are made according to Equations \eqref{eq:upper01}, \eqref{eq:lower01} which are stated in the main text.
\item In Algorithm \ref{alg:active cost-sensitive}, for all $1\leq t\leq T_n$, the refinement condition $A_r^{(t)}(\mathcal{X}_{h,i})=True$ is equivalent to $V(n_{h,i}(t))\leq 2B_{h}$, where $n_{h,i}(t)$ is the number of interactions made with the cell $\mathcal{X}_{h,i}$ up to step $t$.
\end{itemize}
Then, on the event $E$ we have:

\begin{equation}
\label{eq:lemma-largestdepht}
\rho^{h_{\max}(n)}\leq\left\{
 \begin{array}{ll}   
        \left(\frac{\tau}{n}.c_{cs}\log\left(2n^3M\right)\right)^{\frac{1}{2\alpha+d}} \quad\text{if}\; \tau\geq \left(\frac{1}{n}.c_{cs}\log\left(2n^3M\right)\right)^{\frac{\alpha\beta}{2\alpha+d-\alpha\beta}} \\
        \left(\frac{1}{n}.c_{cs}\log\left(2n^3M\right)\right)^{\frac{1}{2\alpha+d-\alpha\beta}}\;\text{if}\;\tau\leq \left(\frac{1}{n}.c_{cs}\log\left(2n^3M\right)\right)^{\frac{\alpha\beta}{2\alpha+d-\alpha\beta}},
    \end{array}
\right.
\end{equation}
where $c_{cs}$ is a constant independent of $n$.
\end{lem}

\begin{proof}~\\
Let $\mathcal{X}_{h,i}$ be an active cell (as introduced in Definition \ref{def:active}). By construction, the Algorithm \ref{alg:active cost-sensitive} does not interact with the cell $\mathcal{X}_{h,i}$ indefinitely, and an upper bound on the number of interactions is determined by the refinement criterion. In this case, the number of interactions with $x_{h,i}$ is then upper bounded as follows for all $t\leq T_n$: 

\begin{equation}
    n_{h,i}(t)\leq \nu_2^{-2}\rho^{-2h\alpha}\log\left(2n^3M\right),
\end{equation}
then, in the event $E$, the label budget $n$ is therefore upper bounded by: 

\begin{equation}
\label{eq:lemma-largestdepht1}
n\leq  2\sum_{h=0}^{h_{max}(n)} N_{act,h}(T_n)\left(\nu_2^{-2}\rho^{-2h\alpha}\log\left(2n^3M\right)\right).
\end{equation}
Besides, we have for $h\geq 1$, by Lemma \ref{lemma-activeset}: 
$$N_{act,h}(T_n)\left(\nu_2^{-2}\rho^{-2h\alpha}\log\left(2n^3M\right)\right)\leq \nu_2^{-2}\log\left(2n^3M\right)\left(C\tau\rho^{-h(2\alpha+d)}+C'\rho^{-h(2\alpha+d-\alpha\beta)}\right).$$

Equation \eqref{eq:lemma-largestdepht1} becomes: 
\begin{align*}n &\leq C_1 \log\left(2n^3M\right)\left(C\sum_{h=0}^{h_{max}(n)}\tau\rho^{-h(2\alpha+d)}+C'\sum_{h=0}^{h_{max}(n)}\rho^{-h(2\alpha+d-\alpha\beta)}\right)\quad \text{where}\,C_1=2\nu_2^{-2}\\
               &\leq C_1 \log\left(2n^3M\right)\left(C_2\tau\rho^{-h_{\max}(n)(2\alpha+d)}+C_3\rho^{-h_{\max}(n)(2\alpha+d-\alpha\beta)}\right)\quad\text{for some constants}\;C_2,C_3>0 \\
               &\leq  C_4\log\left(2n^3M\right)\max\left(\tau\rho^{-h_{\max}(n)(2\alpha+d)},\rho^{-h_{\max}(n)(2\alpha+d-\alpha\beta)}\right)\quad\text{where}\;C_4=2\max(C_3,C_2)C_1.
\end{align*}
If $\tau\rho^{-h_{\max}(n)(2\alpha+d)}\geq \rho^{-h_{\max}(n)(2\alpha+d-\alpha\beta)}$, that is $\tau\geq \rho^{h_{\max}\alpha\beta}$, then we can easily obtain 
$$\rho^{h_{\max}(n)}\leq \left(\frac{\tau}{n}.C_4\log\left(2n^3M\right)\right)^{\frac{1}{2\alpha+d}}$$
Similarly, if $\tau\rho^{-h_{\max}(n)(2\alpha+d)}\leq \rho^{-h_{\max}(n)(2\alpha+d-\alpha\beta)}$, 
we get
$$\rho^{h_{\max}(n)}\leq \left(\frac{1}{n}.C_4\log\left(2n^3M\right)\right)^{\frac{1}{2\alpha+d-\alpha\beta}}.$$
Consequently, 
$$\rho^{h_{\max}(n)}\leq \max\left(\left(\frac{1}{n}.C_4\log\left(2n^3M\right)\right)^{\frac{1}{2\alpha+d-\alpha\beta}}, \left(\frac{\tau}{n}.C_4\log\left(2n^3M\right)\right)^{\frac{1}{2\alpha+d}}\right),$$ 
which can be rewritten as:

\begin{equation}
\rho^{h_{\max}(n)}\leq\left\{
 \begin{array}{ll}   
        \left(\frac{\tau}{n}.C_4\log\left(2n^3M\right)\right)^{\frac{1}{2\alpha+d}} \quad\text{if}\; \tau\geq \left(\frac{1}{n}.C_4\log\left(2n^3M\right)\right)^{\frac{\alpha\beta}{2\alpha+d-\alpha\beta}} \\
        \left(\frac{1}{n}.C_4\log\left(2n^3M\right)\right)^{\frac{1}{2\alpha+d-\alpha\beta}}\;\text{if}\;\tau\leq \left(\frac{1}{n}.C_4\log\left(2n^3M\right)\right)^{\frac{\alpha\beta}{2\alpha+d-\alpha\beta}} 
    \end{array}
\right.
\end{equation}

\end{proof}

\begin{lem}{(monotonicity of the classification uncertainty)}\label{lemma:monotonocity}~\\
For a fixed $t\geq 1$, let $I(t)$  be defined as: 

\begin{equation}
    I(t)=\max_{\mathcal{X}_{h,i}\in\mathcal{X}_u^{(t)}}\;\left(\min_{y\in \mathcal{Y}^{(t)}_{h,i}}\,U_{f(:;y)}^{(t)}(\mathcal{X}_{h,i})-\min_{y\in \mathcal{Y}^{(t)}_{h,i}} L_{f(:;y)}^{(t)}(\mathcal{X}_{h,i})\right),
\end{equation}
where for $y\in \mathcal{Y}^{(t)}_{h,i}$, the quantities $U_{f(:;y)}^{(t)}(\mathcal{X}_{h,i})$ and $L_{f(:;y)}^{(t)}(\mathcal{X}_{h,i})$ are respectively the upper and lower confidence bounds of $f(.;y)$ at step $t.$ Let us assume that 
\begin{itemize}
\item The Assumption \ref{ass:smoothness} holds. 
\item In Algorithm \ref{alg:active cost-sensitive}, the update of upper and lower confidence bounds are made according to Equations \eqref{eq:upper01}, \eqref{eq:lower01} which are stated in the main text.

\end{itemize}
 Then, in the event $E$, the function $I$ is non increasing with respect to $t$.
\end{lem}

\begin{proof}~\\
This follows from the definition of  $U_{f(:;y)}^{(t)}(\mathcal{X}_{h,i})$ and $L_{f(:;y)}^{(t)}(\mathcal{X}_{h,i})$  in Equations \eqref{eq:upper01}, \eqref{eq:lower01}. 
\end{proof}

\begin{lem}{(Rate of convergence)}~\\
Let us assume that Assumptions \ref{ass:smoothness}, \ref{Ass:margin-noise}, \eqref{ass:strong density} hold and that: 
\begin{itemize}
\item In Algorithm \ref{alg:active cost-sensitive}, the update of upper and lower confidence bounds are made according to Equations \eqref{eq:upper01}, \eqref{eq:lower01} which are stated in the main text.
\item In Algorithm \ref{alg:active cost-sensitive}, for all $1\leq t\leq T_n$, the refinement condition $A_r^{(t)}(\mathcal{X}_{h,i})=True$ is equivalent to $V(n_{h,i}(t))\leq 2B_{h}$, where $n_{h,i}(t)$ is the number of interactions made with the cell $\mathcal{X}_{h,i}$ up to step $t$.
\end{itemize}
Then, with probability at least $1-\frac{1}{n}$, the excess risk of the classifier $\hat{g}_n$ provided by Algorithm \ref{alg:active cost-sensitive} satisfies: 
\begin{equation}
\varepsilon_{cs}(\hat{g}_n)\leq\left\{
 \begin{array}{ll}   
        C_6\left(\frac{\tau}{n}.c_{cs}\log\left(2n^3M\right)\right)^{\frac{\alpha(\beta+1)}{2\alpha+d}} \quad\text{if}\; \tau\geq \left(\frac{1}{n}.c_{cs}\log\left(2n^3M\right)\right)^{\frac{\alpha\beta}{2\alpha+d-\alpha\beta}} \\
        C_9\left(\frac{1}{n}.c_{cs}\log\left(2n^3M\right)\right)^{\frac{\alpha(\beta+1)}{2\alpha+d-\alpha\beta}}\;\text{if}\;\tau\leq \left(\frac{1}{n}.c_{cs}\log\left(2n^3M\right)\right)^{\frac{\alpha\beta}{2\alpha+d-\alpha\beta}}, 
    \end{array}
\right.
\end{equation}
where $c_{cs}, C_6, C_9$ are absolute constants.
\end{lem}

\begin{proof}~\\
On the event $E$, the excess can be rewritten as:
\begin{align*}
    \varepsilon_{cs}(\hat{g}_n)&=\mathbb{E}_{X,c}(c(\hat{g}_n)-c(f_{cs}^*(x)))\\
                        &=\mathbb{E}_{X}\left[f(X;\hat{g}_n(X))-f(X;f_{cs}^*(X))\right]\\
                        &=\mathbb{E}_{X}\left[f(X;\hat{g}_n(X))-f(X;f_{cs}^*(X))\mathds{1}_{X\in \bigcup_{\mathcal{X}_{h,i}\in \mathcal{X}_u^{(T_n)}} \mathcal{X}_{h,i}}\mathds{1}_{f(X,\hat{g}_n(X))\neq f(X,f_{cs}^*(X))}\right]\quad \text{by Lemma \ref{lemma:correct}},
\end{align*}
where $T_n$ is introduced in Lemma \ref{lemma-largestdepht}.\\
Let $x$ $\in$  $\mathcal{X}_{h,i}$, with $\mathcal{X}_{h,i}\in \mathcal{X}_u^{(T_n)}$, then on the event $E$, we have by definition of $\hat{g}_n$ and Equation \eqref{eq:boundconfi} 
 \begin{equation}
 \label{eq:lemma:excess1}
     f(x,\hat{g}_n(x))-f(x,f_{cs}^*(x))\leq \min_{y\in \mathcal{Y}^{(T_n)}_{h,i}}\,U_{f(:;y)}^{(T_n)}(\mathcal{X}_{h,i})-\min_{y\in \mathcal{Y}^{(T_n)}_{h,i}} L_{f(:;y)}^{(T_n)}(\mathcal{X}_{h,i})
 \end{equation}
 At step $T_n$,  let $\mathcal{X}_{h_{\max}(n),i_{\max}}$ be denoted as the deepest active cell and $x_{h_{\max}(n),i_{\max}}$ its corresponding center. As introduced in Definition \ref{def:active}, let $\hat{t}(x_{h_{max}})\leq T_n$. To simplify the notations, we denote $\hat{t}:=\hat{t}(x_{h_{max}}).$
 
 Furthermore,  let us consider $\mathcal{X}_{h_{\max}(n)-1,\tilde{i}}$ (for some $\tilde{i}=1,\ldots 2^{h_{\max}(n)-1}$) the cell at the previous depth which contains $x_{h_{\max}(n),i_{\max}}$ and we denote by $x_{h_{\max}(n)-1,\tilde{i}}$ its corresponding center.  Then by Lemma \ref{lemma:monotonocity}, the r.h.s of \eqref{eq:lemma:excess1} can be upper bounded: 
 
 \begin{equation}
 \label{eq:lemma:excess2}
    f(x,\hat{g}_n(x))-f(x,f_{cs}^*(x))\leq \min_{y\in \mathcal{Y}^{(\hat{t})}_{h_{\max}(n)-1,\tilde{i}}}\,U_{f(:;y)}^{(\hat{t})}(\mathcal{X}_{h_{max}(n)-1,\tilde{i}})-\min_{y\in \mathcal{Y}^{(\hat{t})}_{h_{\max}(n)-1,\tilde{i}}}\,L_{f(:;y)}^{(\hat{t})}(\mathcal{X}_{h_{max}(n)-1,\tilde{i}})
 \end{equation}
 As the refinement criterion is satisfied at that time (at the beginning of step $\hat{t}$) at $\mathcal{X}_{h_{\max}(n)-1,\tilde{i}}$, we have: $$V(n_{h_{\max}(n)-1,\tilde{i}}(\hat{t}))\leq 2B_{h_{\max}(n)-1},$$ and Equation \eqref{eq:lemma:excess2} becomes: 
 
 \begin{equation}
 \label{eq:excess3}
     f(x,\hat{g}_n(x))-f(x,f^*(x))\leq 6B_{h_{\max}(n)-1}.
 \end{equation}
 Besides, let $x$ $\in$ $\cup_{\mathcal{X}_{h,i}\in \mathcal{X}_u^{(T_n)}} \mathcal{X}_{h,i}$ with $f(x,\hat{g}_n(x))\neq f(x,f_{cs}^*(x))$, then we have: 
 \begin{align*}
\Delta(x) &\leq f(x,\hat{g}_n(x))-f(x,f^*(x))\\
                                       &\leq 6B_{h_{\max}(n)-1}\quad \text{by Equation}\; \eqref{eq:excess3}.
\end{align*}
 Consequently, by using Assumption \ref{Ass:margin-noise}, we get: 
 
 $$\varepsilon_{cs}(\hat{g}_n)\leq 6B_{h_{\max}(n)-1}P_X\left(\Delta(x)\leq 6B_{h_{\max}(n)-1}\right)$$
 and by Lemma \ref{lemma-largestdepht}, if
 $$\tau\geq \left(\frac{1}{n}.C_4\log\left(2n^3M\right)\right)^{\frac{\alpha\beta}{2\alpha+d-\alpha\beta}},$$ 
 we obtain: 
 \begin{align*}
 \varepsilon_{cs}(\hat{g}_n)&\leq C_6\rho^{h_{\max}(n)\alpha(\beta+1)}\quad\text{where}\; C_6=36C_{\beta}\frac{\nu_2^{\alpha(\beta+1)}}{\rho^{\alpha(\beta+1)}}\\
                     & \leq C_6\left(\frac{\tau}{n}.C_4\log\left(2n^3M\right)\right)^{\frac{\alpha(\beta+1)}{2\alpha+d}}.
 \end{align*}
 
 Similarly, if 
 \begin{equation}
 \label{eq7:proofexcess}
 \tau\leq \left(\frac{1}{n}.C_4\log\left(2n^3M\right)\right)^{\frac{\alpha\beta}{2\alpha+d-\alpha\beta}},
 \end{equation}
 we have: 
 \begin{align*}
  \varepsilon_{cs}(\hat{g})&\leq 6b_{h_{\max}(n)-1}P_X\left(\Delta'(x)\leq 6b_{h_{\max}(n)-1}\right)\\
  &\leq C_7\rho^{h_{\max}(n)\alpha}\left(\tau+C_8\rho^{h_{\max}(n)\alpha\beta}\right)\quad\text{where}\;C_7=6\frac{\nu_2^{\alpha}}{\rho^{\alpha}}, C_8=6C_{\beta}\frac{\nu_2^{\alpha(\beta+1)}}{\rho^{\alpha\beta}}.
 \end{align*}
 By using \eqref{eq7:proofexcess} and \eqref{eq:lemma-largestdepht}, we get: 
 
 \begin{align*}
 \varepsilon_{cs}(\hat{g})&\leq \left(\frac{1}{n}.C_4\log\left(2n^3M\right)\right)^{\frac{\alpha}{2\alpha+d-\alpha\beta}}\left( \left(\frac{1}{n}.C_4\log\left(2n^3M\right)\right)^{\frac{\alpha\beta}{2\alpha+d-\alpha\beta}}+C_8\left(\frac{1}{n}.C_4\log\left(2n^3M\right)\right)^{\frac{\alpha\beta}{2\alpha+d-\alpha\beta}}\right) \\
                  &\leq C_9\left(\frac{1}{n}.C_4\log\left(2n^3M\right)\right)^{\frac{\alpha(\beta+1)}{2\alpha+d-\alpha\beta}}.
 \end{align*}
 
 \end{proof}

\subsection{Proof of Theorem \ref{theo-lower}}
\label{sec:prooftheolower}
This Section is devoted to provide the lower bounds corresponding to the upper bounds provided in Theorem \ref{theo:upper-bounds}. We will firstly state the relation that exists between problems from standard multiclass classification and cost-sensitive classification. By referring to Assumptions respectively from Section \ref{sec:assumption}, this preliminary result will allow us to obtain our lower bound in cost-sensitive setting through a lower bound obtained in standard classification.
\subsubsection{Relation with multi-class classification} 
Let us consider the setting of the standard multi-class classification which is defined as follows: let $\mathcal{Y}$ be the set of labels $\lbrace 1,\ldots, M\rbrace $, and an unknown probability $P$ defined on $\mathcal{X}\times\mathcal{Y}$. We consider an i.i.d sample $S_n=\{(X_1,Y_1),\ldots, (X_n,Y_n)\}$ generated according to the probability $P.$ Based on the sample $S_n$, the aim is to construct a classifier $\hat{g}:\mathcal{X}\longrightarrow \mathcal{Y}$ that minimizes the error of classification 
$$R_{c\ell}(g)=P(g(X)\neq Y).$$  
Let $\eta(x)=(\eta_1(x), \ldots, \eta_M(x))$ be the regression function with $\eta_j(x)=P(Y=j\,\vert X=x)$ for all $j$ 
$\in$ $\lbrace 1,\ldots,M\rbrace$. It is well-known that the oracle mapping $$f_{c\ell}^*(x)\in \argmax_{j\in\lbrace 1,\ldots,M\rbrace}\;\eta_j(x)$$ minimizes the risk $R_{c\ell}(g)$ over all measurable classifiers.\\
For $x$ $\in$ $\mathcal{X}$, and $y\in \mathcal{Y}$, let $\Delta_{\eta}(x,y)$ be defined as: 
$$\Delta_{\eta}(x,y)=\max_{y'\in\mathcal{Y}}\,\eta_{y'}(x)-\eta_y(x),$$ and
$$ \Delta_{\eta}(x)= \left\{
    \begin{array}{ll}
          \min_{y\in\mathcal{Y}}\lbrace \Delta_{\eta}(x,y):\;\Delta_{\eta}(x,y)>0\rbrace\;\;\mbox{If}& \;\;\;\exists\;y\in\mathcal{Y},\;\Delta_{\eta}(x,y)>0 \\
        \infty & \;\;\;\mbox{Otherwise.}
    \end{array}
\right.$$
Given an instance $x$, let us consider the order statistic on $\eta(x):$
$$\eta_{(1)}(x)\geq \eta_{(2)}(x)\geq...\geq \eta_{(M)}(x).$$
We define: 
$$\Delta'_{\eta}(x)=\eta_{(1)}(x)-\eta_{(2)}(x).$$

Equivalently as in the cost-sensitive setting, we introduce consider the following assumptions standard multi-class classification:
 
\begin{ass}[Refined margin noise assumption in multi-class classification]\label{Ass:margin-noise-multiclass}~\\
There exist parameters $\beta_{\eta}, C_{\beta_{\eta}}, C'_{\beta_{\eta}}, \tau\geq 0$ such that for all $\epsilon>0$
$$P_{X}(x\in \mathcal{X},\;\;\Delta_{\eta}(x)\leq \epsilon)\leq C_{\beta_{\eta}}\epsilon^{\beta_{\eta}},$$
and 
$$P_{X}(x\in \mathcal{X},\;\;\Delta'_{\eta}(x)\leq \epsilon)\leq \tau+ C'_{\beta_{\eta}}\epsilon^{\beta_{\eta}}.$$
\end{ass}

\begin{ass}[Smoothness assumption on the regression function]\label{Ass:smoothness assumption-multiclass}~\\
There exist two parameters $\alpha_{\eta}$, $L_{\eta}>0$ such that for all $x$, $z$ $\in$  $\mathcal{X}$, we have 
$$\vert\eta_y(x)-\eta_y(z)\vert\leq L_{\eta}\parallel x-z\parallel_2^{\alpha_{\eta}}\quad \text{for all}\; y\in \mathcal{Y}.$$
\end{ass}
Similarly as in Definition \ref{def:cost-sensitive measure} in the main text, we introduce the class of probabilities in standard classification setting.
\begin{defi}[Multi-class classification measure]\label{def:classification measure}~\\
Let $\Xi=(0,1)\times (0,+\infty)^4\times[0,1]\times (0,1)\times(0,\infty)$. For any $\zeta=(\alpha_{\eta},L_{\eta},\beta_{\eta},C_{\beta_{\eta}},C'_{\beta_{\eta}},\tau,\mu_{\min},\mu_{\max})\in \Xi$, we denote $\mathcal{P}_{c\ell}(\zeta)$ the class of probability measures $P$ on $\mathcal{X}\times\mathcal{Y}$ such that: 
\begin{itemize}
    \item The regression function $($associated to $P)$ $\eta$ satisfies Assumption~\ref{Ass:smoothness assumption-multiclass} with parameters $\alpha_{\eta},L_{\eta}$.
    \item The probability $P$ satisfies Assumption \ref{Ass:margin-noise-multiclass} with parameters $\beta_{\eta},C_{\beta_{\eta}},C'_{\beta_{\eta}},\tau$.
    \item The probability $P$ satisfies Assumption \ref{ass:strong density} with parameters $\mu_{\min},\mu_{\max}$. 
 \end{itemize}
\end{defi}

Below, we state a result from \citep{reeve2019learning} which relates multi-class classification problem to the 
corresponding problems within the cost-sensitive framework.
\begin{prop}[\citep{reeve2019learning}]\label{prop:convert:cost}~\\
Let us suppose we have a classification problem which consists in minimising over all measurable classifiers $g:\mathcal{X} \longrightarrow \mathcal{Y}$  the error of classification $R_{c\ell}(g)=P(g(X)\neq Y)$ given an i.i.d sample $(X_1,Y_1)\ldots, (X_n,Y_n)$ distributed according to a probability $P$ over $\mathcal{X}\times \mathcal{Y}$. To each couple $(X,Y)$, we associate $(X,c)$ with $c=\Gamma^{(c\ell)}\Phi_{\mathcal{Y}}(Y)$  where $\Gamma^{(c\ell)}$ is a $M\times M$ matrix defined by: $\Gamma^{(c\ell)}_{ij}=1$ if $i\neq j$, $\Gamma^{(c\ell)}_{ii}=0$ and $\Phi_{\mathcal{Y}}(Y)$ is a $M\times 1$ matrix with $\Phi_{\mathcal{Y}}(Y)_{j1}=0$ if $j\neq Y$ and $\Phi_{\mathcal{Y}}(Y)_{YY}=1$. This leads to a cost-sensitive problem $P^*$ with the following properties: 

\begin{itemize}
    \item The oracle mapping $f_{cs}^*(x)$ $\in$ $\argmin \lbrace y\in \mathcal{Y}, f(x;y)\rbrace $ where $f(x;y)=\mathbb{E}(c(y)\vert X=x)$.
    \item Given any classifier $g:\mathcal{X} \longrightarrow \mathcal{Y}$, we have: 
    \begin{align*}
    \varepsilon_{cs}(g)&=\mathbb{E}\left[f(X;g(X))-f(X;f_{cs}^*(X))\right]\\
    &=\varepsilon_{c\ell}(g)=\mathbb{E}\left[\eta_{f_{c\ell}^{*}(X)}(X)-\eta_{g(X)}(X)\right].
    \end{align*}
    \item Let $\zeta=(\alpha_{\eta},L_{\eta},\beta_{\eta},C_{\beta_{\eta}},C'_{\beta_{\eta}},\tau,\mu_{\min},\mu_{\max})\in \Xi$ introduced in Definition \ref{def:classification measure}. Whenever the classification problem $P$ belongs to $\mathcal{P}_{c\ell}(\zeta)$, the corresponding cost-sensitive problem $P^*$ belongs to $\mathcal{P}_{cs}(\zeta)$. 
\end{itemize}
\end{prop}
The above proposition is very important as it allows to translate multi-class classification problem to a corresponding one in cost-sensitive framework; in this case, providing a lower bound in multi-class classification  implies a lower bound in the cost-sensitive framework.   

Next, we state our result which provides a lower bound in active learning classification. For simplicity, we will use $M=2$, but the extension to $M>2$ is straightforward.  
\subsubsection{Lower bounds for standard multi-class classification}

\begin{theo}[Lower bound for multi-class classification]\label{theo-lower-class}
Let us consider a multi-class classification problem with $M=2$. Let $\zeta=(\alpha,L,\beta,C_{\beta},C'_{\beta},\tau,\mu_{\min},\mu_{\max})\in \Xi$ introduced in Definition \ref{def:classification measure}. We assume that $L,C_{\beta},C'_{\beta} \in (1,\infty)$,  $\alpha\beta \leq d$ and $\mu_{min}\in (0,\tau)$. There exist  constants $C_{c\ell}$, $C_{c\ell}'$  (independent of $n$) such that for any  active classifier $\hat{g}_n$, we have: 
$$\inf_{\hat{g}_n}\sup_{P\in\mathcal{P}_{c\ell}(\zeta)} \mathbb{E} (R_{c\ell}(\hat{g}_n))-R_{c\ell}(g^*))\geq  \min(a_{n,\tau},a'_{n,\tau}),$$

where 
$$a_{n,\tau}=C_{c\ell}\max\left(\left(\frac{\tau}{n}\right)^{\frac{\alpha(\beta+1)}{2\alpha+d}}, \left(\frac{1}{n}\right)^{\frac{\alpha(\beta+1)}{2\alpha+d-\alpha\beta}}\right)$$
and 
$$a'_{n,\tau}=C'_{c\ell}
\max\left(\left(\frac{1}{n}\right)^{\frac{\alpha(\beta+1)}{2\alpha+d}}, \left(\frac{1}{\tau}\right)^{\frac{\alpha(\beta+1)}{2\alpha+d}}\left(\frac{1}{n}\right)^{\frac{\alpha(\beta+1)}{2\alpha+d-\alpha\beta}}\right).$$
\end{theo}

The proof of Theorem \ref{theo-lower-class} consists in using some standard tools for minimax lower bounds strategy where the aim is to firstly reduce the risk to a multiple hypothesis problem, and thereafter apply Theorem 2.5 from \citep{tsybakov2009springer} presented below. Our proof is mainly based on the use of some universal tools such as packing set. Consequently, it can be extended to general metric spaces.

\begin{theo}[\citep{tsybakov2009springer}, Theorem 2.5]\label{theo:tsy}~\\
Let $\Theta$ be a class of models and $\tilde{d}:\Theta\times\Theta\longrightarrow \mathbb{R}$ a pseudo metric defined on $\Theta$. Let $\{P_g, \; g\in \Theta\}$ be a collection of probability measures associated with $\Theta$. Let us assume there exists a subset $\{g_0,\ldots, g_{\tilde{m}}\}\subset \Theta$, with $\tilde{m}>1$ such that:

\begin{itemize}
    \item $\tilde{d}(g_i,g_j)>2s>0$ for all $0\leq i<j\leq \tilde{m},$
    \item $P_{g_j}\ll P_{g_0}$ for every $1\leq j\leq \tilde{m},$
    \item $\frac{1}{\tilde{m}}\sum_{j=1}^{\tilde{m}} KL\left(P_{g_j},P_{g_0}\right)\leq \gamma \log(\tilde{m})$, where $0<\gamma<\frac{1}{8},$ and  $KL(.,.)$ is the Kullback-Leibler divergence. 
\end{itemize}
Then, 
$$\inf_{\hat{g}}\sup_{g\in \Theta}\;P_{g}\left(\tilde{d}(\hat{g},g)>s\right)\geq \frac{\sqrt{\tilde{m}}}{1+\sqrt{\tilde{m}}}\left(1-2\gamma-\sqrt{\frac{2\gamma}{\log(\tilde{m})}}\right),$$
where the infimum is taken over all possible estimators based on a sample generated from $P_g.$
\end{theo}

\textbf{Proof of Theorem \ref{theo-lower-class}}
\\
\\
The proof consists in applying Theorem \ref{theo:tsy} to a suitable family of distributions $P$ from $ \mathcal{P}(\zeta)$. In doing so, we will proceed in several steps as follows: 

\begin{enumerate}
    \item \textbf{Construction of the family of measures}

    Let $\bar{x}=(\frac{1}{2},\ldots,\frac{1}{2})$ and $\bar{r} \in (0,1)$ small enough, for example $\bar{r}=\frac{1}{128}.$ Let $r\in (0,1)$ small enough, with $r\leq \bar{r}$ and $\{x_1,\ldots, x_{Q(r)}\}$ a (maximal) packing set of $B(\bar{x},\bar{r}).$  Let $m(r)\leq Q(r)$ be a positive quantity. Let us consider the subset $ \{x_1,\ldots, x_{m(r)}\}$, and the set $$S(r)=\bigcup_{i=1}^{m(r)}\bar{B}(x_i,\frac{r}{2}),$$

where $B(x,r)=\lbrace z\in \mathcal{X},\;\|x-z\|< r\rbrace$ and $\bar{B}(x,r)=\lbrace z\in \mathcal{X},\;\|x-z\|\leq r\rbrace$.\\
    Let $r^*$ be defined as: 
    \begin{equation}
    \label{eq:r*}
    r^*=\inf\{r_0>0,\;S(r)\subset \bar{B}(\bar{x},r_0)\}.
    \end{equation}
    The quantity $r$ will be chosen small enough in order to have $\bar{B}(\bar{x},r^*)\subset [0,1]^d$. Let $a$ $\in [0,1]^d$, $c_0$, $c_1$ $\in$ $(0,1)$ small enough and the Euclidean ball $B(a,c_0r^{\frac{\alpha\beta}{d}})$ such that 
    $$B(a,c_1r^{\frac{\alpha\beta}{d}})\bigcap B(\bar{x},r^*+c_0r^{\frac{\alpha\beta}{d}})=\emptyset.$$
    Let $\phi$ be the function defined as: 
    \begin{align*}
    \phi\colon&[0,\,+\infty)\longrightarrow [0,1]\\
    &\phantom{++++}x\mapsto\min((2-3x)_+,1),
     \end{align*}
     where $(z)_+=\max(z,0)$ for $z\in \mathbb{R}.$ The function $\phi$ is 3-Lipschitz.\\
     As  $\{x_1,\ldots,x_{m(r)}\}$ is a $r-$ separated set, we have that $B(x_i,\frac{r}{2}),\;i=1\ldots,m(r)$ are disjoints sets. We thus define the following functions for all $i$ $\in$ $\{1,\ldots, m(r)\}$ and $r$ small enough: 
     \begin{align*}
    f_i\colon&[0,\,+\infty)\longrightarrow [0,1]\\
    &\phantom{++++}x\mapsto \frac{c_2}{d^{\frac{d}{2\beta}}}Lr^{\alpha}\phi\left(\frac{2}{r}\parallel x-x_i\parallel\right)\mathds{1}_{B(x_i,\frac{r}{2})},
     \end{align*}
     Where $c_2\in (0,\tfrac{1}{12L})$ is small enough. 
     Finally, for all $\sigma$ $\in$ $\{-1,1\}^{m(r)}$, we define: 
     
     \begin{equation}
         \label{eq:regression-lower}
         \eta_{\sigma}(x) = \left\{
    \begin{array}{ll}
        \frac{1}{2}+\sum_{i=1}^{m(r)} \sigma_if_i(x)& \mbox{if } \;x \in \bar{B}(\bar{x},r^*) \\
        \frac{1}{2} + \frac{L}{d^{\frac{d}{2\beta}}}.dist(x,\bar{B}(\bar{x},r^*))^{\frac{d}{\beta}} & \mbox{if}\; x\in \bar{B}(\bar{x},r^*+c_0r^{\frac{\alpha\beta}{d}})\setminus \bar{B}(\bar{x},r^*).\\
        \frac{1}{2}+ \frac{L}{d^{\frac{d}{2\beta}}}c_0^{d/\beta}r^{\alpha} & \mbox{elsewhere}
    \end{array}
\right.
     \end{equation}
where $dist(x,A)=\inf\{\parallel x-y\parallel,\; y\in A\}$ for $x\in [0,1]^d$, and $A\subset [0,1]^d$. To ensure the smoothness condition, we will chose $c_2\leq\frac{1}{12L}$ and $c_0\leq\min\left(\frac{1}{12L},\left(\frac{1}{4L}\right)^{\beta/d}\right)$.    
 
 Besides, we define the marginal probability through the density: 
 \begin{equation}
         \label{eq:marginal-lower}
         p(x) = \left\{
    \begin{array}{ll}
        \frac{w}{Vol(B(x_i,\frac{r}{6}))} & \mbox{if } \;x \in \bar{B}(x_i,\frac{r}{6}), i=1,\ldots,m(r). \\
        \frac{\tau}{Vol(A_1)} & \mbox{if}\; x\in A_1.\\
        \frac{1-m(r)w-\tau}{Vol(A_2)} & \mbox{if}\; x\in A_2.\\
        0 &\mbox{elsewhere}.
    \end{array}
\right.
     \end{equation}
     
     Where $0<w\leq \frac{1}{m(r)}$, $\tau$ comes from the Assumption \ref{Ass:margin-noise}, $A_1=\bar{B}(\bar{x},r^*)\setminus \bigcup\limits_{i=1}^{m(r)} B(x_i,\frac{r}{3})$, and  $A_2=\bar{B}(a,c_0r^{\frac{\alpha\beta}{d}}).$
     
  \item \textbf{Proof of the smoothness condition} 
  : we prove that the family of regression functions defined above is well defined and satisfies the smoothness condition \ref{Ass:smoothness assumption-multiclass}. The proof is closely related to the one provided in \citep{reeve2019learning}.\\
  By definition, the regression  $\eta_{\sigma}$ is well-defined; in fact, $\eta_{\sigma}(x)\in [0,1]$ for all $x\in [0,1]^d$.\\ 
   Let $x,x' \in [0,1]^d$. We firstly assume that $x,x'$ $\in$ $\bar{B}(\bar{x},r^*)$. If $\eta_{\sigma}(x)=\eta_{\sigma}(x')=\frac{1}{2}$, obviously, we have $\vert\eta_{\sigma}(x)-\eta_{\sigma}(x')\vert\leq L\parallel x-x'\parallel^{\alpha}.$
     
      Without loss of generality, let us assume $\eta_{\sigma}(x)\neq \frac{1}{2}$. Then $x$ $\in$ $\bar{B}(x_i,\frac{r}{3})$ for some $i=1,\ldots, m(r)$. If $x'$ $\in$ $B(x_i,\frac{r}{2})$, then we have: 
      
      \begin{align*}
    \vert\eta_{\sigma}(x)-\eta_{\sigma}(x')\vert &\leq 6c_2\frac{L}{r} r^{\alpha}\left(\parallel x-x_i\parallel-\parallel x'-x_i\parallel\right)\\
           & \leq  Lr^{\alpha-1}\parallel x-x'\parallel\\
           & \leq L\parallel x-x'\parallel^{\alpha}\quad \text{as}\; \parallel x-x'\parallel\leq r\;\text{and}\; \alpha\leq 1.
      \end{align*}
      If $x'$ $\notin B(x_i,\frac{r}{2})$, we have: 
      $$\parallel x-x'\parallel\geq \parallel x_i-x'\parallel-\parallel x_i-x\parallel \geq \frac{r}{6}.$$ Consequently, 
      \begin{align*}          \vert\eta_{\sigma}(x)-\eta_{\sigma}(x')\vert &=c_2Lr^{\alpha}\phi\left(\frac{2}{r}\parallel x-x_i\parallel\right)\\
      & \leq L\parallel x-x'\parallel^{\alpha}
      \end{align*}

       If $x'$ $\in$ $ \bar{B}(\bar{x},r^*+c_0r^{\frac{\alpha\beta}{d}})\setminus \bar{B}(\bar{x},r^*)$ and $x$ $\in$ $\bar{B}(\bar{x},r^*)$ with $\eta(x)\neq \frac{1}{2}$, then $x$ $\in$ $B(x_i,\frac{r}{3})$ for some $i=1,\ldots,m(r).$ In this case, we have: 
      $$\parallel x-x'\parallel\geq \frac{r}{6}.$$
      Consequently,
  \begin{align*}
      \vert \eta_{\sigma}(x)-\eta_{\sigma}(x') \vert &\leq  c_2Lr^{\alpha}\phi\left(\frac{2}{r}\parallel x-x_i\parallel\right)+ L.dist(x,\bar{B}(\bar{x},r^*))^{\frac{d}{\beta}}\\
      &\leq L\parallel x-x'\parallel^{\alpha}.
   \end{align*}
   If $x'$ $\in$ $ \bar{B}(\bar{x},r^*+c_0r^{\frac{\alpha\beta}{d}})\setminus \bar{B}(\bar{x},r^*)$ and $\eta(x)= \frac{1}{2}$, then we have: 
  \begin{align*}
      \vert \eta_{\sigma}(x)-\eta_{\sigma}(x') \vert &\leq L.dist(x',\bar{B}(\bar{x},r^*))^{\frac{d}{\beta}}\\
      &\leq L\parallel x-x'\parallel^{\alpha}\quad.
    \end{align*}  
  If $x, x'$ $\in$ $\bar{B}(\bar{x},r^*+c_0r^{\frac{\alpha\beta}{d}})\setminus \bar{B}(\bar{x},r^*)$, by using the fact that $z\mapsto dist(z,\bar{B}(\bar{x},r^*))$ is Lipschitz, we have: 
  $$\vert\eta_{\sigma}(x)-\eta_{\sigma}(x') \vert \leq L\parallel x-x'\parallel^{\alpha}.$$
  
  If $x'$ $\in$ $\bar{B}(\bar{x},r^*+c_0r^{\frac{\alpha\beta}{d}})\setminus \bar{B}(\bar{x},r^*)$ and $x$ $\in$ $[0,1]^d\setminus\bar{B}(\bar{x},r^*+c_0r^{\frac{\alpha\beta}{d}})$ then we have: 
  $$dist(x,\bar{B}(\bar{x},r^*))\geq c_0r^{\frac{\alpha\beta}{d}}\geq dist(x',\bar{B}(\bar{x},r^*))$$
  we have thus: 
  \begin{align*}
      \vert \eta_{\sigma}(x)-\eta_{\sigma}(x') \vert &=\frac{L}{d^{\frac{d}{2\beta}}}c_0^{d/\beta}r^{\alpha}- \frac{L}{d^{\frac{d}{2\beta}}}.dist(x',\bar{B}(\bar{x},r^*))^{\frac{d}{\beta}}\\
      &\leq \frac{L}{d^{\frac{d}{2\beta}}}(dist(x,\bar{B}(\bar{x},r^*)^{d/\beta}-dist(x',\bar{B}(\bar{x},r^*)^{d/\beta})\\
      &\leq \frac{L}{d^{\frac{d}{2\beta}}}\parallel x-x'\parallel^{d/\beta}\\
      &\leq L\parallel x-x'\parallel^{\alpha}.
  \end{align*}

If $x$, $x'$ $\in$ $[0,1]^d\setminus\bar{B}(\bar{x},r^*+c_0r^{\frac{\alpha\beta}{d}})$, obviously, we have: 
$$\vert \eta_{\sigma}(x)-\eta_{\sigma}(x') \vert\leq L\parallel x-x'\parallel^{\alpha}.$$ 
\item \textbf{Proof of the margin noise condition:} we prove that the family $(\eta_{\sigma},P_X)$ satisfies Assumption \ref{Ass:margin-noise-multiclass}. 
This is crucial as it allows us to provide an effective choice of $m(r)$ and $w$. 
On the support of $P_X$ except the set $A_1$,  we have:  $\vert\eta(x)-\frac{1}{2}\vert\geq c_3r^{\alpha} $, where $c_3=\frac{L}{d^{\frac{d}{2\beta}}}\min(c_2,c_0^{d/\beta})$. Therefore, for $t>1$, we have: 
\begin{align*}
    P_X(x\in [0,1]^d,\;0<\vert \eta(x)-\frac{1}{2}\vert< c_3tr^{\alpha}) &\leq m(r)w+ P_X(x\in A_2,\; 0<\vert \eta(x)-\frac{1}{2}\vert< c_3tr^{\alpha})\\
    & \leq m(r)w+P_X(B(a,c_1t^{\beta/d}r^{\alpha\beta/d})) \quad \text{as}\;t\geq 1\\
    &\leq m(r)w+ c_4 (tr^{\alpha})^{\beta}\quad \text{for some constant }\;c_4\\        \end{align*}
By taking \begin{equation}
\label{eq:mw}
m(r)w=c_5r^{\alpha\beta},
\end{equation}
with $c_5$ is small enough, we get:   
\begin{equation}
\label{eq:tsy-bayes-far}
P_X(x\in [0,1]^d,\;0<\vert \eta(x)-\frac{1}{2}\vert< c_3tr^{\alpha}) \leq C_{\beta}\left(c_3tr^{\alpha}\right)^{\beta}.
\end{equation}

Besides, we have: 
\begin{equation}
\label{eq:bayes-boundary}
P_X(\eta_{\sigma}(x)=\frac{1}{2})=P_X(A_1)=\tau.
\end{equation}
Finally, by combining \eqref{eq:tsy-bayes-far} and \eqref{eq:bayes-boundary}, we have that Assumption \ref{Ass:margin-noise-multiclass} is satisfied.

Moreover, we can easily check that Assumption \ref{ass:strong density} (with a suitable choice of $m$ and $w$, which will be stated later) is also satisfied. 

\item \textbf{Application of Theorem \ref{theo:tsy}} 

In order to apply Theorem \ref{theo:tsy}, we proceed as follows: 

\begin{itemize}
    \item We choose $\Theta=\{\eta_{\sigma},\; \sigma\in \{-1,1\}^{m(r)}\}.$
    \item In order to satisfy the first condition in Theorem \ref{theo:tsy}, we construct in a wise way a well-separated subset of $\{-1,1\}^{m(r)}$ by invoking the Gilbert-Varshamov's Lemma \citep[Lemma 2.9]{tsybakov2009springer}: assuming $m(r)>8$, there exists a subset $\{\sigma^0,\ldots,\sigma^{\tilde{m}(r)}\}$ of $\{-1,1\}^{m(r)}$ such that: 
    
    $$\sigma^{0}=(1,\ldots,1)$$ and for all $i\neq j$, 
    \begin{equation}
    \label{eq:separated}
    d_{H}(\sigma^{i},\sigma^{j})>\frac{m(r)}{8},\quad\quad \tilde{m}(r)\geq 2^{m(r)/8},
    \end{equation}
    where $d_H$ stands for the Hamming distance.
    \item We define the pseudo distance on $\tilde{\Theta}=\{\eta_{\sigma^0},\ldots,\eta_{\sigma^{\tilde{m}(r)}}\}$: for all $\sigma$, $\sigma'\in$ $\{\sigma^0,\ldots,\sigma^{\tilde{m}(r)}\}$,  
    
    $$\tilde{d}(\eta_{\sigma},\eta_{\sigma'})=P_X\left(\sign(\eta_{\sigma}-\frac{1}{2})\neq \sign(\eta_{\sigma'}-\frac{1}{2})\right).$$
  By using \eqref{eq:separated}, we have for all $\sigma$ $\neq$ $\sigma'\in$ $\{\sigma^0,\ldots,\sigma^{\tilde{m}(r)}\}$, 
  \begin{equation}
      \label{eq:separated1}
     \tilde{d}(\eta_{\sigma},\eta_{\sigma'})\geq \frac{m(r)w}{8}. 
  \end{equation}
  \item Next, for $\sigma$ $\in$ $\{\sigma^0,\ldots,\sigma^{\tilde{m}(r)}\}$,  let us consider the corresponding probability $P_{\sigma}$ with regression function $\eta_{\sigma}$ $\in$ $\tilde{\Theta}$, and marginal probability $P_X$. Following the Lemma 1 in \citep{castro2008minimax}, and Equation 10 in \citep{minsker2012plug}, we have: 
  \begin{equation}
  \label{eq:KL}
      KL(P_{\sigma,n},P_{\sigma^0,n})\leq 32nc_5r^{2\alpha}
  \end{equation}
  where $$c_5=2\left(\frac{c_2}{d^{\frac{d}{2\beta}}}L\right)^2$$ and $P_{\sigma,n}$ is the joint probability of sample $(X_1,Y_1),\ldots, (X_n,Y_n)$ provided by any active learning algorithm, with $(X_i,Y_i)\sim P_{\sigma}$ for all $1\leq i\leq n.$
  
  We thus have to choose $m(r),r$ appropriately in order to have: 
  
  $$KL(P_{\sigma,n},P_{\sigma^0,n})\leq 32nc_5r^{2\alpha}\leq \gamma m(r)\quad \text{with}\;0<\gamma<\frac{1}{8}.$$

  In that case, by \eqref{eq:separated}, we obtain: 
  $$\frac{1}{\tilde{m}}\sum_{j=1}^{\tilde{m}} KL\left(P_{g_j},P_{g_0}\right)\leq \gamma \log(\tilde{m}.)$$
  By setting $m(r)=\lfloor c_6 r^{\alpha\beta-d} \rfloor$ and $$r=c_7\max\left(\left(\frac{\tau}{n}\right)^{\frac{1}{2\alpha+d}}, \left(\frac{1}{n}\right)^{\frac{1}{2\alpha+d-\alpha\beta}}\right).$$
  
  Assuming $\tau\leq n^{-\frac{\alpha\beta}{2\alpha+d-\alpha\beta}}$,  we get: 
  
 $$32nc_5r^{2\alpha}\leq \gamma' m(r),$$ 
 for some constant $\gamma'$. The constants $c_6, c_7$ will be chosen such that $0<\gamma'< \frac{1}{8}$, $m(r)>8, $ and $m(r)\leq Q(r).$
 
 Let us take $w=c_8r^{d}$, with $c_8$ small enough in order to satisfy \eqref{eq:mw}, and make the marginal probability well-defined. In that case, following \eqref{eq:separated1}, we have for all $\sigma\neq \sigma'$ $\in$ $\{\sigma^0,\ldots,\sigma^{\tilde{m(r)}}\}$: 
 $$\tilde{d}(\eta_{\sigma},\eta_{\sigma'})\geq c_9\max\left(\left(\frac{\tau}{n}\right)^{\frac{\alpha\beta}{2\alpha+d}}, \left(\frac{1}{n}\right)^{\frac{\alpha\beta}{2\alpha+d-\alpha\beta}}\right).$$
 We can conclude that for any active learning algorithm $\hat{\eta}:$ 
 
 $$\underset{\hat{\eta}}{inf}\;\underset{\sigma\in\{\sigma_0,\ldots,\sigma^{\tilde{m}(r)}\}}{sup}\; \mathbb{P}\left(P_X(\sign(\eta_{\sigma}(x)-\tfrac{1}{2})\neq\sign(\hat{\eta}(x)-\tfrac{1}{2})\geq c_{10}\max\left(\left(\frac{\tau}{n}\right)^{\frac{\alpha\beta}{2\alpha+d}}, \left(\frac{1}{n}\right)^{\frac{\alpha\beta}{2\alpha+d-\alpha\beta}}\right)\right)\geq \frac{1}{8},$$
 and then for any active learning algorithm $\hat{\eta}$ with corresponding classifier $g_{\hat{\eta}}$, we have by using Proposition 3.1 from \citep{chzhen2021minimax}:
 
  $$\underset{\hat{\eta}}{inf}\;\underset{P\in \mathcal{P}_{cl}}{sup}\; \mathbb{P}\left(R_{cl,P}(g_{\hat{\eta}})-R_{cl,P}(g^*)\geq c_{11}\max\left(\left(\frac{\tau}{n}\right)^{\frac{\alpha(\beta+1)}{2\alpha+d}}, \left(\frac{1}{n}\right)^{\frac{\alpha(\beta+1)}{2\alpha+d-\alpha\beta}}\right)\right)\geq \frac{1}{8},$$
and finally, by Markov's inequality, we get the required lower bound: 
$$\underset{\hat{\eta}}{inf}\;\underset{P\in \mathcal{P}_{cl}}{sup}\; \mathbb{E}\left(R_{cl,P}(g_{\hat{\eta}})-R_{cl,P}(g^*)\right)\geq \frac{c_{11}}{8}\max\left(\left(\frac{\tau}{n}\right)^{\frac{\alpha(\beta+1)}{2\alpha+d}}, \left(\frac{1}{n}\right)^{\frac{\alpha(\beta+1)}{2\alpha+d-\alpha\beta}}\right).$$

In the case $\tau\geq n^{-\frac{\alpha\beta}{2\alpha+d-\alpha\beta}}$, the proof is nearly similar
to the proof with $\tau\leq n^{-\frac{\alpha\beta}{2\alpha+d-\alpha\beta}}$, differing only on the bounding of
the Kullback Leibler divergence and the choice of parameters $r$, $m(r)$, $w$.
Instead, we choose $r, m(r), w$ such that: 
$$r=c_{12}\max\left(\left(\frac{1}{n}\right)^{\frac{1}{2\alpha+d}}, \left(\frac{1}{\tau}\right)^{\frac{1}{2\alpha+d}}\left(\frac{1}{n}\right)^{\frac{1}{2\alpha+d-\alpha\beta}}\right),$$
$m(r)=\lfloor c_{13}r^{-d}\rfloor$, $w=c_{14}r^{d+\alpha\beta}$
with $c_{12}$, $c_{14}$ small enough, and $c_{13}$ chosen such that $m(r)\leq Q(r)$, $m(r)>8$.
 \end{itemize}
\end{enumerate}

\subsubsection{Proof of Theorem \ref{theo-lower}}
Theorem \ref{theo-lower-class} states that there are some constants $C_{cl}$, $C'_{cl}$ (independent of $n$), such that any active learning classification algorithm $\hat{g}$ which provides a sample $S_n:=(X_i,Y_i)_{1\leq i\leq n}$, and thus a classifier $\hat{g}_{n,c\ell}$ based on $S_n$, there exists some $\mathbb{P}$ $\in$ $\mathcal{P}_{c\ell}(\zeta)$ with: 
\begin{equation}
\label{eq:convert11}
    \mathbb{E}_{\sim \mathbb{P}^n}(R_{c\ell}(\hat{g}_{n,c\ell}))-R_{c\ell}(g^*))\geq\min(a_{n,\tau},a'_{n,\tau}).
\end{equation}

Now, let us suppose we have an active learning algorithm $\hat{g}$, which, through a sampling strategy, provides a sample $S_{n,c\ell}=:(X_i,Y_i)_{1\leq i\leq n}$, and then a classifier $\hat{g}_{n,c\ell}$ based on $S_n$. For some $\mathbb{P}$ $\in$ $\mathcal{P}_{c\ell}(\zeta)$, Equation \eqref{eq:convert11} holds. Moreover, the active learning algorithm $\hat{g}$ can be converted into an active learning algorithm for cost-sensitive classification as follows: 
\begin{itemize}
    \item Consider the same sampling strategy. 
    \item After getting the sample $(X_t,Y_t)$ at step $t$, we convert it into a cost-sensitive vector $(X_t,c_t)$ following the procedure considered in Proposition \ref{prop:convert:cost}.
    
    \end{itemize}
Finally, we obtain the sample $S_{n,cs}:=(X_i,c_i)_{1\leq i\leq n}$, and we provide the cost-sensitive classifier defined as $\hat{g}_{n,cs}:=\hat{g}_{n,c\ell}.$
As considered in Proposition \ref{prop:convert:cost}, the corresponding probability $P^*$ such that $((X_t,Y_t)\sim P$ implies $(X_t,c_t)\sim P^*)$ belongs to $\mathcal{P}_{cs}(\zeta)$ and again, by Proposition \ref{prop:convert:cost}, we have: 
\begin{align*}
    \mathbb{E}_{\sim (P*)^n}(R_{cs}(\hat{g}_{n,cs}))-R_{cs}(g^*))&=\mathbb{E}_{\sim P^n}(R_{c\ell}(\hat{g}_{n,c\ell}))-R_{c\ell}(g^*))\\
    &\geq\min(a_{n,\tau},a'_{n,\tau}).
\end{align*} 
Thus, we get the same lower bounds in the cost-sensitive learning: 
$$\inf_{\hat{g}_n}\sup_{P\in\mathcal{P}_{cs}(\zeta)} \mathbb{E} (R_{cs}(\hat{g}_n))-R_{cs}(g^*))\geq  \min(a_{n,\tau},a'_{n,\tau}).$$

\bibliographystyle{alt24}
\bibliography{alt24}

\begin{thebibliography}{22}
\providecommand{\natexlab}[1]{#1}
\providecommand{\url}[1]{\texttt{#1}}
\providecommand{\urlprefix}{URL }
\expandafter\ifx\csname urlstyle\endcsname\relax
  \providecommand{\doi}[1]{doi:\discretionary{}{}{}#1}\else
  \providecommand{\doi}{doi:\discretionary{}{}{}\begingroup \urlstyle{rm}\Url}\fi

\bibitem[{Agarwal(2013)}]{agarwal2013selective}
Agarwal, A. (2013).
\newblock Selective sampling algorithms for cost-sensitive multiclass prediction.
\newblock In \emph{International Conference on Machine Learning}, pp. 1220--1228. PMLR.

\bibitem[{Audibert \& Tsybakov(2007)}]{audibert2007fast}
Audibert, J.-Y. \& Tsybakov, A.~B. (2007).
\newblock Fast learning rates for plug-in classifiers.
\newblock \emph{The Annals of statistics} \textbf{35}, 608--633.

\bibitem[{Balcan \emph{et~al.}(2007)Balcan, Broder \& Zhang}]{balcan2007margin}
Balcan, M.-F., Broder, A. \& Zhang, T. (2007).
\newblock Margin based active learning.
\newblock In \emph{International Conference on Computational Learning Theory}, pp. 35--50. Springer.

\bibitem[{Castro \& Nowak(2008)}]{castro2008minimax}
Castro, R.~M. \& Nowak, R.~D. (2008).
\newblock Minimax bounds for active learning.
\newblock \emph{IEEE Transactions on Information Theory} \textbf{54}, 2339--2353.

\bibitem[{Chzhen \emph{et~al.}(2021)Chzhen, Denis \& Hebiri}]{chzhen2021minimax}
Chzhen, E., Denis, C. \& Hebiri, M. (2021).
\newblock Minimax semi-supervised set-valued approach to multi-class classification.
\newblock \emph{Bernoulli} \textbf{27}, 2389--2412.

\bibitem[{Cohn \emph{et~al.}(1994)Cohn, Atlas \& Ladner}]{cohn1994improving}
Cohn, D., Atlas, L. \& Ladner, R. (1994).
\newblock Improving generalization with active learning.
\newblock \emph{Machine learning} \textbf{15}, 201--221.

\bibitem[{Dasgupta(2011)}]{dasgupta2011two}
Dasgupta, S. (2011).
\newblock Two faces of active learning.
\newblock \emph{Theoretical computer science} \textbf{412}, 1767--1781.

\bibitem[{Dmochowski \emph{et~al.}(2010)Dmochowski, Sajda \& Parra}]{dmochowski2010maximum}
Dmochowski, J.~P., Sajda, P. \& Parra, L.~C. (2010).
\newblock Maximum likelihood in cost-sensitive learning: Model specification, approximations, and upper bounds.
\newblock \emph{Journal of Machine Learning Research} \textbf{11}.

\bibitem[{Elkan(2001)}]{elkan2001foundations}
Elkan, C. (2001).
\newblock The foundations of cost-sensitive learning.
\newblock In \emph{International joint conference on artificial intelligence}, vol.~17, pp. 973--978. Lawrence Erlbaum Associates Ltd.

\bibitem[{Gentile \& Orabona(2012)}]{gentile2012multilabel}
Gentile, C. \& Orabona, F. (2012).
\newblock On multilabel classification and ranking with partial feedback.
\newblock \emph{Advances in Neural Information Processing Systems} \textbf{25}.

\bibitem[{Hanneke \& Yang(2015)}]{hanneke2015minimax}
Hanneke, S. \& Yang, L. (2015).
\newblock Minimax analysis of active learning.
\newblock \emph{The Journal of Machine Learning Research} \textbf{16}, 3487--3602.

\bibitem[{Kpotufe \emph{et~al.}(2022)Kpotufe, Yuan \& Zhao}]{kpotufe2022nuances}
Kpotufe, S., Yuan, G. \& Zhao, Y. (2022).
\newblock Nuances in margin conditions determine gains in active learning.
\newblock In \emph{International Conference on Artificial Intelligence and Statistics}, pp. 8112--8126. PMLR.

\bibitem[{Krishnamurthy \emph{et~al.}(2017)Krishnamurthy, Agarwal, Huang, Daum{\'e}~III \& Langford}]{krishnamurthy2017active}
Krishnamurthy, A., Agarwal, A., Huang, T.-K., Daum{\'e}~III, H. \& Langford, J. (2017).
\newblock Active learning for cost-sensitive classification.
\newblock In \emph{International Conference on Machine Learning}, pp. 1915--1924. PMLR.

\bibitem[{Krishnamurthy \emph{et~al.}(2019)Krishnamurthy, Agarwal, Huang, III \& Langford}]{JMLR:v20:17-681}
Krishnamurthy, A., Agarwal, A., Huang, T.-K., III, H.~D. \& Langford, J. (2019).
\newblock Active learning for cost-sensitive classification.
\newblock \emph{Journal of Machine Learning Research} \textbf{20}, 1--50.
\newblock \urlprefix\url{http://jmlr.org/papers/v20/17-681.html}.

\bibitem[{Locatelli \emph{et~al.}(2017)Locatelli, Carpentier \& Kpotufe}]{locatelli2017adaptivity}
Locatelli, A., Carpentier, A. \& Kpotufe, S. (2017).
\newblock Adaptivity to noise parameters in nonparametric active learning.
\newblock In \emph{Proceedings of the 2017 Conference on Learning Theory, PMLR}.

\bibitem[{Minsker(2012)}]{minsker2012plug}
Minsker, S. (2012).
\newblock Plug-in approach to active learning.
\newblock \emph{Journal of Machine Learning Research} \textbf{13}.

\bibitem[{Munos \emph{et~al.}(2014)}]{munos2014bandits}
Munos, R. \emph{et~al.} (2014).
\newblock From bandits to monte-carlo tree search: The optimistic principle applied to optimization and planning.
\newblock \emph{Foundations and Trends{\textregistered} in Machine Learning} \textbf{7}, 1--129.

\bibitem[{Njike \& Siebert(2022)}]{njike2022multi}
Njike, B.~N. \& Siebert, X. (2022).
\newblock Multi-class classification in nonparametric active learning.
\newblock In \emph{International Conference on Artificial Intelligence and Statistics}, pp. 7124--7162. PMLR.

\bibitem[{Reeve \& Brown(2017)}]{reeve2017minimax}
Reeve, H.~W. \& Brown, G. (2017).
\newblock Minimax rates for cost-sensitive learning on manifolds with approximate nearest neighbours.
\newblock In \emph{International Conference on Algorithmic Learning Theory}, pp. 11--56.

\bibitem[{Reeve(2019)}]{reeve2019learning}
Reeve, H. W.~J. (2019).
\newblock \emph{Learning in high dimensions with asymmetric costs}.
\newblock The University of Manchester (United Kingdom).

\bibitem[{Shekhar \emph{et~al.}(2021)Shekhar, Ghavamzadeh \& Javidi}]{shekhar2021active}
Shekhar, S., Ghavamzadeh, M. \& Javidi, T. (2021).
\newblock Active learning for classification with abstention.
\newblock \emph{IEEE Journal on Selected Areas in Information Theory} \textbf{2}, 705--719.

\bibitem[{Tsybakov(2009)}]{tsybakov2009springer}
Tsybakov, A.~B. (2009).
\newblock Springer series in statistics.

\end{thebibliography}
\end{document}